\documentclass[journal]{IEEEtran}
\usepackage{cite} 

\usepackage{url}
\usepackage{listings}
\usepackage{amssymb}
\usepackage{graphicx}
\usepackage{algcompatible}
\usepackage{algorithm}
\usepackage{url}
\usepackage{rotating}
\usepackage{booktabs}
\usepackage{subfig}
\usepackage{amsmath}
\usepackage{bbm}

\renewcommand{\labelenumi}{(\alph{enumi})}
\renewcommand\theenumi\labelenumi

\usepackage[english]{babel}

\usepackage[utf8]{inputenc}
\usepackage{xspace}
\usepackage{amsmath,amsthm,amssymb,mathtools}
\usepackage{lmodern}

\usepackage[algo2e,ruled,vlined,linesnumbered]{algorithm2e}

\usepackage{xcolor}
\usepackage{tikz}
\usepackage{graphicx}

\allowdisplaybreaks[4]
\newtheorem{theorem}{Theorem}
\newtheorem{lemma}[theorem]{Lemma}

\newtheorem{definition}[theorem]{Definition}

\newcommand{\oea}{$(1 + 1)$~EA\xspace}

\newcommand{\mpoea}{$(\mu+1)$~EA\xspace}

\newcommand{\R}{\ensuremath{\mathbb{R}}}

\newcommand{\N}{\ensuremath{\mathbb{N}}} 
\newcommand{\Z}{\ensuremath{\mathbb{Z}}}

\begin{document}
\sloppy

\title{Analysis of Evolutionary Algorithms on Fitness Function with Time-linkage Property}

\author{Weijie~Zheng, Huanhuan~Chen, and~Xin~Yao%
\thanks{This work was supported by Guangdong Basic and Applied Basic Research Foundation (Grant No. 2019A1515110177), Guangdong Provincial Key Laboratory (Grant No. 2020B121201001), the Program for Guangdong Introducing Innovative and Enterpreneurial Teams (Grant No. 2017ZT07X386), Shenzhen Science and Technology Program (Grant No. KQTD2016112514355531), the Program for University Key Laboratory of Guangdong Province (Grant No. 2017KSYS008), National Natural Science Foundation of China (Grant No. 61976111), and Science and Technology Innovation Committee Foundation of Shenzhen (Grant No. JCYJ20180504165652917). (Corresponding author: Xin Yao.)}
\thanks{Weijie Zheng is with Guangdong Provincial Key Laboratory of Brain-inspired Intelligent Computation, Research Institute of Trustworthy Autonomous Systems (RITAS), Department of Computer Science and Engineering, Southern University of Science and Technology, Shenzhen, China. He is also with School of Computer Science and Technology, University of Science and Technology of China, Hefei, China.}
\thanks{Huanhuan Chen is with School of Computer Science and Technology, University of Science and Technology of China, Hefei, China.}
\thanks{Xin Yao is with Guangdong Provincial Key Laboratory of Brain-inspired Intelligent Computation, Research Institute of Trustworthy Autonomous Systems (RITAS), Department of Computer Science and Engineering, Southern University of Science and Technology, Shenzhen, China. He is also with CERCIA, School of Computer Science, University of Birmingham, Birmingham, United Kingdom.}
}

\maketitle

\begin{abstract}
In real-world applications, many optimization problems have the time-linkage property, that is, the objective function value relies on the current solution as well as the historical solutions. Although the rigorous theoretical analysis on evolutionary algorithms has rapidly developed in recent two decades, it remains an open problem to theoretically understand the behaviors of evolutionary algorithms on time-linkage problems. This paper takes the first step to rigorously analyze evolutionary algorithms for time-linkage functions. Based on the basic OneMax function, we propose a time-linkage function where the first bit value of the last time step is integrated but has a different preference from the current first bit. We prove that with probability $1-o(1)$, randomized local search and \oea cannot find the optimum, and with probability $1-o(1)$, $(\mu+1)$ EA is able to reach the optimum. 

%
\end{abstract}


\section{Introduction}
Evolutionary Algorithms (EAs), one category of stochastic optimization algorithms that are inspired by Darwinian principle and natural selection, have been widely utilized in real-world applications. Although EAs are simple and efficient to use, the theoretical understandings on the working principles and complexity of EAs are much more complicated and far behind the practical usage due to the difficulty of mathematical analysis caused by their stochastic and iterative process. 

In order to fundamentally understand EAs and ultimately design efficient algorithms in practice, researchers begin the rigorous analysis on functions with simple and clear structure, majorly like pseudo-Boolean function and classic combinatorial optimization problem, like in the theory books~\cite{NeumannW10, AugerD11, Jansen13, ZhouYQ19, DoerrN20}. Despite the increasing attention and insightful theoretical analyses in recent decades, there remain many important open areas that haven't been considered in the evolutionary theory community. 

One important open issue is about the time-linkage problem. Time-linkage problem, firstly introduced by Bosman~\cite{Bosman05} into the evolutionary computation community, is the optimization problem where the objective function to be optimized relies not only on the solutions of the current time but also the historical ones. In other words, the current decisions also influence the future. There are plenty of applications with time-linkage property. We just list the dynamic vehicle routing with time-varying locations to visit~\cite{Bosman05} as a slightly detailed example. Suppose that the locations are clustered. Then the current vehicle serving some locations in one cluster is more efficient to serve other locations in the same cluster instead of serving the currently available locations when the locations oscillate among different clusters in future times. Besides, the efficiency of the current vehicle routing would influence the quality of the service, which further influences the future orders, that is, future locations to visit. In a word, the current routing and the impact of it in the future together determine the income of this company. The readers could also 
refer to the survey in~\cite{Nguyen11} to see more than $30$ real-world applications, like an optimal watering scheduling to improve the quality of the crops along with the weather change~\cite{MorimotoOSB07}.

The time-linkage optimization problems can be tackled offline or online according to different situations. If the problem pursues an overall solution with sufficient time budget and time-linkage dynamics can be integrated into a static objective function, then the problem can be solved offline. However, in the theoretical understanding on the static problem~\cite{NeumannW10, AugerD11, Jansen13, ZhouYQ19, DoerrN20}, no static benchmark function in the evolutionary theory community is time-linkage. 

Another situation that the real-world applications often encounter is that the solution must be solved online as the time goes by. This time-linkage online problem belongs to dynamic optimization problem~\cite{Nguyen11}. As pointed out in~\cite{Nguyen11}, the whole evolutionary community, not only the evolutionary theory community, is lack of research on these real-world problems. The dynamic problem analyzed so far in the theory community majorly includes Dynamic OneMax~\cite{Droste02}, Magnitude and Balance~\cite{RohlfshagenLY09}, Maze~\cite{KotzingM12}, Bi-stable problem~\cite{JansenZ15}, dynamic linear function~\cite{LenglerS18}, and dynamic BinVal function~\cite{LenglerM20} for dynamic pseudo-Boolean function, and dynamic combinatorial problems including single-destination shortest path problem~\cite{LissovoiW15}, makespan scheduling~\cite{NeumannW15}, vertex cover problem~\cite{PourhassanGN15}, subset selection~\cite{RoostapourNNF19}, graph coloring~\cite{BossekNPS19}, etc. However, there is no theoretical analysis on dynamic time-linkage fitness function, even no dynamic time-linkage pseudo-Boolean function is proposed for the theoretical analysis.


The main contributions of this paper can be summarized as follows. This paper conducts the first step towards the understanding of EAs on the time-linkage function. When solving a time-linkage problem by EAs in an offline mode, the first thing faced by the practitioners to utilize EAs is how they encode the solution. There are obviously two straightforward encoding ways. Take the objective function relying on solutions of two time steps as an example. One way is to merely ignore the time-linkage dependency by solving a non-time-linkage function with double problem size. The other way is to consider the time-linkage dependency, encode the solution with the original problem size, but store the solutions generated in the previous time steps for the fitness evaluation. 
When solving the time-linkage problem in an online mode, the engineers need to know before they conduct the experiments whether the algorithm they use can solve the problem or not. Hence, in this paper, we design a time-linkage toy function based on OneMax to shed some light on these questions. This function, called OneMax$_{(0,1^n)}$ where $n$ is the dimension size, is the sum of two components, one is OneMax fitness of the current $n$-dimensional solution, the other one is the value of the first dimension in the previous solution but multiplying minus dimension size. The design of this function considers the situation when the current solution prefers a different value from the previous solution, which could better show the influence of different encodings. Also, it could be the core element of some dynamic time-linkage functions and used in the situation that each time step we only optimize the current state of the online problem in a limited time, so that the analysis of this function could also show some insights to the undiscovered theory for the dynamic time-linkage function. 

For our results, this paper analyzes the theoretical behaviors of randomized local search (RLS) and two most common benchmark EAs, \oea and $(\mu+1)$ EA, on OneMax$_{(0,1^n)}$. We will show that with probability $1-o(1)$, RLS and \oea cannot find the optimum of OneMax$_{(0,1^n)}$ (\textbf{Theorem~\ref{thm:11EAOM}}) while the not small population size in $(\mu+1)$ EA can help it reach the optimum with probability $1-o(1)$ (\textbf{Theorem~\ref{thm:mpoEAOM}}). We also show that conditional on an event with probability $1-o(1)$, the expected runtime for $(\mu+1)$ EA is $O(n\mu)$ (\textbf{Theorem~\ref{thm:runtimempoEA}}).

The remainder of this paper is organized as following. In Section~\ref{sec:om01}, we introduce the motivation and details about the designed OneMax$_{(0,1^n)}$. Section~\ref{sec:oea} shows the theoretical results of RLS and \oea on OneMax$_{(0,1^n)}$, and our theoretical results of $(\mu+1)$ EA are shown in Section~\ref{sec:muoea}. Our conclusion is summarized in Section~\ref{sec:con}.

\section{OneMax$_{(0,1^n)}$ Function}
\label{sec:om01}
\subsection{OneMax$_{(0,1^n)}$ Function}
For the first time-linkage problem for theoretical analysis, we expect the function to be simple and with clear structure. OneMax, which counts the total ones in a bit string, is considered as one of the simplest pseudo-Boolean functions, and is a well-understood benchmark in the evolutionary theory community on static problems. Choosing it as a base function to add the time-linkage property could facilitate the theoretical understanding on the time-linkage property. Hence, the time-linkage function we will discuss in this paper is based on OneMax. In OneMax function, each dimension has the same importance and the same preference for having a dimension value $1$. We would like to show the difference, or more aggressively show the difficulty that the time-linkage property will cause, which could better help us understand the behavior of EAs on time-linkage problems. Therefore, we will introduce the solutions of the previous steps but with different importance and preference. For simplicity of analysis, we only introduce one dimension, let's say the first dimension, value of the last time step into the objective function but with the weight of $-n$, where $n$ is the dimension size. Other weights could also be interesting, but for the first time-linkage benchmark, we just take $-n$ to show the possible difficulty caused by the time-linkage property. More precisely, this function $f:\{0,1\}\times\{0,1\}^n \rightarrow \Z$ is defined by
\begin{align}
f(x^{ {t-1}},x^{ {t}})=\sum_{i=1}^n x_i^{ {t}}-nx_1^{ {t-1}}
\label{eq:oms}
\end{align}
for two consecutive $x^{ {t-1}}=(x_1^{ {t-1}},\dots,x_n^{ {t-1}})$ and $x^{ {t}}=(x_1^{ {t}},\dots,x_n^{ {t}}) \in \{0,1\}^n$. Clearly, (\ref{eq:oms}) consists of two components, OneMax component relying on the current individual, and the drawing-back component determined by the first bit value of the previous individual. If our goal is to maximize (\ref{eq:oms}), it is not difficult to see that the optimum is unique and the maximum value $n$ is reached if and only if $(x_1^{ {t-1}},x^{ {t}})=(0,1^n)$. Hence, we integrate $(0,1^n)$ and call (\ref{eq:oms}) OneMax$_{(0,1^n)}$ function.

The maximization of the proposed OneMax$_{(0,1^n)}$ function specializes the maximization of the more general time-linkage pseudo-Boolean problems ${h: \{0,1\}^n\times\dots\times\{0,1\}^n \rightarrow \R}$ defined by
\begin{align}
h(x^{ {t_0}},\dots,x^{ {t_0+\ell}})=\sum_{t=0}^{\ell} h_t(x^{ {t_0+t}}; x^{ {t_0}},\dots,x^{ {t_0+t-1}})
\label{eq:h}
\end{align}
for consecutive $x^{ {t_0}},x^{ {t_0+1}},\dots,x^{ {t_0+\ell}}$ where $\ell \in \N$ and could be infinite. OneMax$_{(0,1^n)}$ function could be regarded as a specialization with $\ell=1$, 
$h_0(x^{ {t_0}})=-nx_1^{ {t_0}}$, and $h_1(x^{ {t_0+1}};x^{ {t_0}})=\sum_{i=1}^n x_i^{ {t_0+1}}$.
We acknowledge that more complicated models are more interesting and practical, like with $\ell >1$, with more than one bit value and with other weight values for the historical solutions, etc., but current specialization facilitates the establishment of the first theoretical analysis for the time-linkage problems, and our future work will focus on the analyses of more complicated models.

\subsection{Some Notes}
\label{subsec:notes}
Time-linkage optimization problem can be solved offline or online due to different situations. In the following, we follow the main terminology from~\cite{Bosman05}, the first paper that used the term ``time-linkage'' and introduced it into the evolutionary computation community.
\subsubsection{Offline mode}
Solving the general time-linkage problem $h$ defined above in an offline mode means that we could evaluate all possible $(x^{ {t_0}},x^{ {t_0+1}},\dots,x^{ {t_0+\ell}})$ before determine the final solution for the problem $h$. In this case, the optimum is defined differently when we use different representations. Obviously, there are two straightforward kinds of representations. One is ignoring the time-linkage fact and encoding $(x^{ {t_0}},x^{ {t_0+1}},\dots,x^{ {t_0+\ell}})$ into an $n(\ell+1)$-bit string as one solution, and the optimum is a search point with $n(\ell+1)$ dimensions. We denote this optimum as $X^*$. Since the problem is transferred to a traditional non-time-linkage problem, it is not of interest of our topic. 

The other kind of representation is considering the time-linkage property and encoding $(x^{ {t'}},x^{ {t'+1}},\dots,x^{ {t_0+\ell}}),  {t'>t_0}$, into an $m$-bit string, $m \in  {\{n,2n,\dots,\ell n\}}$, and storing other historical solutions for objective function evaluation. In this case, the optimum is a search point with $m$ dimensions taking the same values as the last $m$ bit values of $X^*$, the optimum in the first kind of representation, condition on the stored solutions taking the same values as the corresponding bit values of $X^*$. For the considered OneMax$_{(0,1^n)}$ function, we encode an $n$-bit string as one solution and store the previous result for objective function evaluation, and the optimum is the current search point with all 1s condition on that the stored previous first bit has value 0. This representation is more interesting to us since now $h$ and OneMax$_{(0,1^n)}$ function are truly time-linkage functions and we could figure out how EAs react to the time-linkage property. Hence, the later sections only consider this representation when solving the OneMax$_{(0,1^n)}$ function in an offline mode is analyzed.

\subsubsection{Online mode}
As in \cite[Section 2]{Bosman05}, the online mode means that we regard it as a dynamic optimization problem, and that the process continues only when the decision on the current solution is made, that is, solutions cannot be evaluated for any time $t>t^{now}$ and we can only evaluate the quality of the historical and the present solutions. More precisely, at present time $t^{now}$, we can evaluate
\begin{align*}
\tilde{h}(x^{ {t_0}},\dots,x^{t^{now}})=\sum_{t=0}^{t^{now}- {t_0}} \tilde{h_t}(x^{ {t_0+t}}; x^{ {t_0}},\dots,x^{ {t_0+t-1}})
\end{align*}
where we note that $\tilde{h_t}$ could dynamically change and could be different from $h_t,t=0,\dots,t^{now}$ in (\ref{eq:h}) when the process ends at $t= {t_0+\ell}$, since the impact of the historical solutions could change when time goes by. Usually, for the online mode, the overall optimum within the whole time period as in the offline mode cannot be reached once some non-optimal solution is made in one time step. Hence, our goal is to obtain the function value as larger as possible before the end of the time period, or more specifically, to obtain some function value above one certain threshold. We notice the similarity to the discounted total reward in the reinforcement learning~\cite[Chapter~3]{SuttonB18}. However, as pointed in~\cite[Section 1]{Bosman05}, the online dynamic time-linkage problem is fundamentally different since in each time the decisions themselves are needed to be optimized and can only be made once, while in reinforcement learning the policies are the solutions and the decisions during the process serve to help determine a good policy. 

Back to the OneMax$_{(0,1^n)}$, the function itself is not suitable to be solved in an online mode since it only contains the previous time step and the current step. 
However, we could still relate OneMax$_{(0,1^n)}$ function to online dynamic optimization problems, via regarding it as one piece of the objective function that considers the overall results during a given time period and each time step we only optimize the current piece. For example, we consider the following dynamic problem
\begin{align}
h(x, {t}) = \max_{x} \sum_{ {\tau=2}}^{ {t}} e^{ {-t+\tau-1}} x^{ {\tau-2}}_1 - n x^{ {t-1}}_1 + \sum_{i=1}^n x^{ {t}}_i
\label{eq:sumom}
\end{align}
where ${x=\{x^2,\dots,x^{ {t}}\}}$, ${x^{ {\tau}}=(x_1^{ {\tau}},\dots,x_n^{ {\tau}}) \in \{0,1\}^n}$ for ${ {\tau=0,1,\dots,t}}$ and the initial $x^0$ and $x^1$ are given. 
For (\ref{eq:sumom}), our goal is to find the solution at some time step $ {t}$ when its function value is greater than $n-1$. Since the previous elements in $1,\dots, {t-2}$ time steps can contribute at most $\sum_{ {\tau=2}}^{ {t}} e^{ {-t+\tau-1}} \le 1/(e-1)$ value, the goal can be transferred to find the time step when the component of the current and the last step, that is, OneMax$_{(0,1^n)}$, has the value of $n$. Thus if we take the strategy for online optimization that we optimize the present each time as discussed in \cite[Section~3]{Bosman05}, that is, for the current time $ {t_{cur}}$, we optimize $h(x,  {t_{cur}})$ with knowing $x^0,\dots,x^{ {t_{cur}-1}}$, then the problem can be functionally regarded as maximizing OneMax$_{(0,1^n)}$ function with $n$-bit string encoding as time goes by. Hence, we could reuse the optimum of OneMax$_{(0,1^n)}$ function in the offline mode as our goal for solving (\ref{eq:sumom}) in an online mode, and call it ``optimum'' for this online mode with no confusion. 


In summary, considering the  OneMax$_{(0,1^n)}$ problem, we note that for the representation encoding $n$-bit string in an offline manner and for optimizing present in an online dynamic manner, the algorithm used for these two situations are the same but with different backgrounds and descriptions of the operators. The details will be discussed when they are mentioned in Sections~\ref{sec:oea} and~\ref{sec:muoea}.

\section{RLS and \oea Cannot Find the Optimum}
\label{sec:oea}
\subsection{RLS and \oea Utilized for OneMax$_{(0,1^n)}$}
\label{subsec:oea}
\oea is the simplest EA that is frequently analyzed as a benchmark algorithm in the evolutionary theory community, and randomized local search (RLS) can be regarded as the simplification of \oea and thus a pre-step towards the theoretical understanding of \oea. Both algorithms are only with one individual in their population. Their difference is on the mutation. In each generation, \oea employs the bit-wise mutation on the individual, that is, each bit is independently flipped with probability $1/n$, where $n$ is the problem size, while RLS employs the one-bit mutation, that is, only one bit among $n$ bits is uniformly and randomly chosen to be flipped. For both algorithms, the generated offspring will replace its parent as long as it has at least the same fitness as its parent. 

The general RLS and \oea are utilized for non-time-linkage function, and they do not consider how we choose the individual representation and do not consider the requirement to make a decision in a short time. We need some small modifications on RLS and \oea to handle the time-linkage OneMax$_{(0,1^n)}$ function. The first issue, the representation choice, only happens when the problem is solved in an offline mode. As mentioned in Section~\ref{subsec:notes}, for the two representation options, we only consider the one that encodes the current solution and stores the previous solution for fitness evaluation. In this encoding, 
we set that only offspring with better or equivalent fitness could affect the further fitness, hoping the optimization process to learn or approach to the situations which are suitable for the time-linkage property.
Algorithm~\ref{alg:1p1EA} shows our modified \oea and RLS for solving OneMax$_{(0,1^n)}$,
and we shall still use the name \oea and RLS in this paper with no confusion. In this case, the optimum of OneMax$_{(0,1^n)}$ is the $1^n$ as the current solution with the stored first bit value of the last generation being $0$. Practically, some termination criterion is utilized in the algorithms when the practical requirement is met. Since we aim at theoretically analyzing the time to reach the optimum, we do not set the termination criterion here. 
\begin{algorithm}[!ht]
    \caption{\oea /RLS to maximize fitness function $f$ requiring two consecutive time steps}
    {\small
    \begin{algorithmic}[1]
    \STATE {Generate the random initial two generations $X^0=(X_{1}^0,\dots,X_{n}^0)$ and $X^1=(X_{1}^1,\dots,X_{n}^1)$}
    \FOR {$g=1,2,\dots$}
    \STATEx {\quad$\%\%$ \textit{Mutation}}
    \STATE {For \oea, generate $\tilde{X}^g$ via independently flipping each bit value of $X^g$ with probability $1/n$;}
    \STATEx \quad{For RLS, generate $\tilde{X}^g$ via flipping one uniformly and randomly selected bit value of $X^g$}
    \STATEx {\quad$\%\%$ \textit{Selection}}
    \IF {$f(X^{g-1},X^g) > f(X^g,\tilde{X}^g)$}
    \STATE {$(X^g,X^{g+1})=(X^{g-1},X^g)$}
    \ELSE
    \STATE {$(X^g,X^{g+1})=(X^g,\tilde{X}^g)$}
    \ENDIF
    \ENDFOR
    \end{algorithmic}
    \label{alg:1p1EA}
    }
\end{algorithm}


The second issue, the requirement to make a decision in a short time, happens when the problem is solved in an online mode. Detailedly, consider the  {problem (\ref{eq:sumom})} we discussed in Section~\ref{subsec:notes} that in each time step we just optimize the present. If the time to make the decision is not so small that \oea or RLS can solve the $n$-dimension problem (OneMax function), then we could obtain $X^t=(1,\dots,1)$ in each time step $t$. Obviously, in this case, the sequence of $\{X^1,\dots,X^{t}\}$ we obtained will lead to a fitness less than $1$ for any time step $t$, and thus we cannot achieve our goal and this case is not interesting. Hence, we assume the time to make the decision is small so that we cannot solve $n$-dimensional OneMax function, and we can just expect to find some result with better or equivalent fitness value each time step. That is, we utilize \oea or RLS to solve OneMax$_{(0,1^n)}$ function, and the evolution process can go on only if some offspring with better or equivalent fitness appears. In this case, we can reuse Algorithm~\ref{alg:1p1EA}, but to note that the generation step $g$ need not to be the same as the time step $t$ of the fitness function since \oea or RLS may need more than one generation to obtain an offspring with better or equivalent fitness for one time step.

In a word, no matter utilizing \oea and RLS to solve OneMax$_{(0,1^n)}$ offline or online, in the theoretical analysis, we only consider Algorithm~\ref{alg:1p1EA} and the optimum is the current search point with all 1s condition on that the stored previous first bit has value 0, without mentioning the solving mode and regardless of the explanations of the different backgrounds.

\subsection{Convergence Analysis of RLS and \oea on OneMax$_{(0,1^n)}$}
\label{subsec:oeacon}
This subsection will show that with high probability RLS and \oea cannot find the optimum of OneMax$_{(0,1^n)}$. Obviously, OneMax$_{(0,1^n)}$ has two goals to achieve, one is to find all 1s in the current string, and the other is to find the optimal pattern $(0,1)$ in the first bit that the current first bit value goes to 1 when the previous first bit value is 0. The two goals are somehow contradictory, so that only one individual in the population of RLS and \oea will cause poor fault tolerance. Detailedly, as we will show in Theorem~\ref{thm:11EAOM}, with high probability, one of twos goal will be achieved before the optimum is found, but the population cannot be further improved. 

Since this paper frequently utilizes different variants of Chernoff bounds, to make it self-contained, we put them from~\cite{Doerr11,Doerr20} into Lemma~\ref{lem:chernoff}. Besides, before establishing our main result, with the hope that it might be beneficial for further research, we also discuss in Lemma~\ref{lem:improvefit} the probability estimate of the event that the increase number of ones from one parent individual to its offspring is $1$ under the condition that the increase number of ones is positive in one iteration of \oea, given the parent individual has $a$ zeros. 

\begin{lemma}[\cite{Doerr11,Doerr20}]
Let $\xi_1,\xi_2, \dots, \xi_m$ be independent random variables. Let $\Xi = \sum_{i=1}^m \xi_i$.
\begin{itemize}
\item[(a)] If $\xi_i$ takes values in $[0,1]$, then for all $\delta\in[0,1]$, ${\Pr[\Xi \le (1-\delta)E[\Xi]] \le \exp(-\delta^2E[\Xi]/2)}$.
\item[(b)] If $\xi_i$ takes values in an interval of length $c_i$, then for all $\lambda \ge 0$, 
$\Pr[\Xi \ge E[\Xi] + \lambda] \le \exp(-2\lambda^2/\sum_{i=1}^m c_i).$
\item[(c)] If $\xi_i$ follows the geometric distribution with success probability $p$ for all $i\in[1..m]$, then for all $\delta \ge 0$, 
$\Pr[\Xi \ge (1+\delta) E[\Xi]] \le \exp\left(-\frac{\delta^2(m-1)}{2(1+\delta)}\right);$
for all $\delta \in [0,1]$,
$\Pr[\Xi \le (1-\delta) E[\Xi]] \le \exp\left(-\frac{\delta^2 m}{2-\frac 43 \delta}\right).$
\end{itemize} 
\label{lem:chernoff}
\end{lemma}

\begin{lemma}
Suppose $X\in\{0,1\}^n$. $Y\in\{0,1\}^n$ is generated by independently flipping each bit of $X$ with probability $\frac 1n$. Let $a\in[1..n]$ be the number of zeros in $X$, and let $|X|$ denote the number of ones in $X$ and $|Y|$ for the number of ones in $Y$. Then 
\begin{align*}
\Pr[|Y|&-|X|=1 \mid |Y|>|X|] \\
={}& \frac{\sum_{i=1}^a \binom{a}{i}\binom{n-a}{i-1}\frac{1}{n^{2i-1}}(1-\frac 1n)^{n-2i+1}}{\sum_{i=1}^a\sum_{j=0}^{i-1}\binom{a}{i}\binom{n-a}{j}\frac{1}{n^{i+j}}(1-\frac{1}{n})^{n-i-j}} > 1-\frac{ea}{n}.
\end{align*}
\label{lem:improvefit}
\end{lemma}
\begin{proof}
Since 
\begin{align*}
\Pr&[|Y|>|X|] = \sum_{i=1}^a\sum_{j=0}^{i-1}\binom{a}{i}\binom{n-a}{j}\frac{1}{n^{i+j}}\left(1-\frac{1}{n}\right)^{n-i-j}\\
\Pr&[|Y|-|X|=1]=\sum_{i=1}^a \binom{a}{i}\binom{n-a}{i-1}\frac{1}{n^{2i-1}}\left(1-\frac 1n\right)^{n-2i+1},
\end{align*}
and ${\Pr[\left(|Y|-|X|=1\right) \cap \left(|Y|>|X|\right)] = \Pr[|Y|-|X|=1]}$, we have
\begin{align*}
\Pr[|Y|{}&-|X|=1 \mid |Y|>|X|] \\
={}& \frac{\Pr[(|Y|-|X|=1) \cap (|Y|>|X|)]}{\Pr[|Y|>|X|]}\\
={}& \frac{\sum_{i=1}^a \binom{a}{i}\binom{n-a}{i-1}\frac{1}{n^{2i-1}}(1-\frac 1n)^{n-2i+1}}{\sum_{i=1}^a\sum_{j=0}^{i-1}\binom{a}{i}\binom{n-a}{j}\frac{1}{n^{i+j}}(1-\frac{1}{n})^{n-i-j}}.
\end{align*}
For $a=1$, it is easy to see that 
\begin{align*}
\Pr[|Y|-|X|=1 \mid |Y|>|X|] = 1 > 1-\frac{e}{n}. 
\end{align*} 
For $a\ge2$, since
\begin{align*}
\sum_{i=2}^a&\sum_{j=0}^{i-2}\binom{a}{i}\binom{n-a}{j}\frac{1}{n^{i+j}}\left(1-\frac{1}{n}\right)^{n-i-j}\\
={}&\sum_{i=2}^a\binom{a}{i}\frac{1}{n^{i}} {\left(1-\frac{1}{n}\right)^{n-i}}\sum_{j=0}^{i-2}\binom{n-a}{j}\frac{1}{n^{j}} {\left(1-\frac{1}{n}\right)^{-j}}\\
\le {}& \sum_{i=2}^a\binom{a}{i}\frac{1}{n^{i}} {\left(1-\frac{1}{n}\right)^{n-i}}\sum_{j=0}^{i-2}\frac{(n-a)^j}{j!}\frac{1}{(n-1)^j} \\
\le {}&\sum_{i=2}^a\binom{a}{i}\frac{1}{n^{i}} {\left(1-\frac{1}{n}\right)^{n-i}} (i-1)\\
= {}& \sum_{i=2}^a \frac{a!}{i!(a-i)!} \frac{1}{n^{i}} {\left(1-\frac{1}{n}\right)^{n-i}}(i-1) \\
< {}&\sum_{i=2}^a\frac{a!}{(i-1)!(a-i)!} \frac{1}{n^{i}} {\left(1-\frac{1}{n}\right)^{n-i}}(i-1) \\
= {}& \sum_{i=2}^a\frac{a!}{(i-2)!(a-i)!}\frac{1}{n^{i}} {\left(1-\frac{1}{n}\right)^{n-i}} \\
= {}& \frac{a(a-1)}{n^2} \sum_{i=2}^a \binom{a-2}{i-2} \frac{1}{n^{i-2}} {\left(1-\frac{1}{n}\right)^{n-i}}\\
\le {}& \frac{a(a-1)}{n^2} \sum_{i=2}^a \binom{a-2}{i-2} \frac{1}{n^{i-2}} {\left(1-\frac{1}{n}\right)^{a-i}} \\
= {}&\frac{a(a-1)}{n^2}  {\left(\frac{1}{n}+1-\frac{1}{n}\right)^{a-2}} = \frac{a(a-1)}{n^2}
\end{align*}
and
\begin{align*}
\sum_{i=1}^a\sum_{j=0}^{i-1}{}&\binom{a}{i}\binom{n-a}{j}\frac{1}{n^{i+j}} {\left(1-\frac{1}{n}\right)^{n-i-j}} \\
\ge{}& \frac{a}{n} {\left(1-\frac{1}{n}\right)^{n-1}} > \frac{a}{en},
\end{align*}
we have
\begin{align*}
\frac{\sum_{i=2}^a\sum_{j=0}^{i-2}\binom{a}{i}\binom{n-a}{j}\frac{1}{n^{i+j}}(1-\frac{1}{n})^{n-i-j}}{\sum_{i=1}^a\sum_{j=0}^{i-1}\binom{a}{i}\binom{n-a}{j}\frac{1}{n^{i+j}}\left(1-\frac{1}{n}\right)^{n-i-j}} 
< \frac{\frac{a(a-1)}{n^2}}{\frac{a}{en}} < \frac{ea}{n}.
\end{align*}
With
\begin{align*}
\sum_{i=1}^a \binom{a}{i}{}&\binom{n-a}{i-1}\frac{1}{n^{2i-1}} {\left(1-\frac 1n\right)^{n-2i+1}} \\
= {}&\sum_{i=1}^a\sum_{j=0}^{i-1}\binom{a}{i}\binom{n-a}{j}\frac{1}{n^{i+j}} {\left(1-\frac{1}{n}\right)^{n-i-j} }\\
{}&-\sum_{i=2}^a\sum_{j=0}^{i-2}\binom{a}{i}\binom{n-a}{j}\frac{1}{n^{i+j}} {\left(1-\frac{1}{n}\right)^{n-i-j}},
\end{align*}
the lemma is proved.
\end{proof}

Now we show the behavior for RLS and \oea optimizing OneMax$_{(0,1^n)}$ function. 
The outline to establish our main theorem (Theorem~\ref{thm:11EAOM}) is shown in Figure~\ref{fig:relation}. 
First, Lemma~\ref{lem:stuck} shows two cases once one of them happens before the optimum is reached, both RLS and \oea cannot find the optimum in arbitrary further generations. Then for the non-trivial behaviors for three different initial states, Lemma~\ref{lem:00}, Lemma~\ref{lem:10}, and Lemma~\ref{lem:11} respectively show that the algorithm will get stuck in one of the two cases. 
In the following, all proofs don't specifically distinguish RLS and \oea due to their similarity, and discuss each algorithm independently only when they have different behaviors. We start with one definition that will be frequently used in our proofs.

\begin{figure}[!ht]
\centering
\includegraphics[width=3.45in]{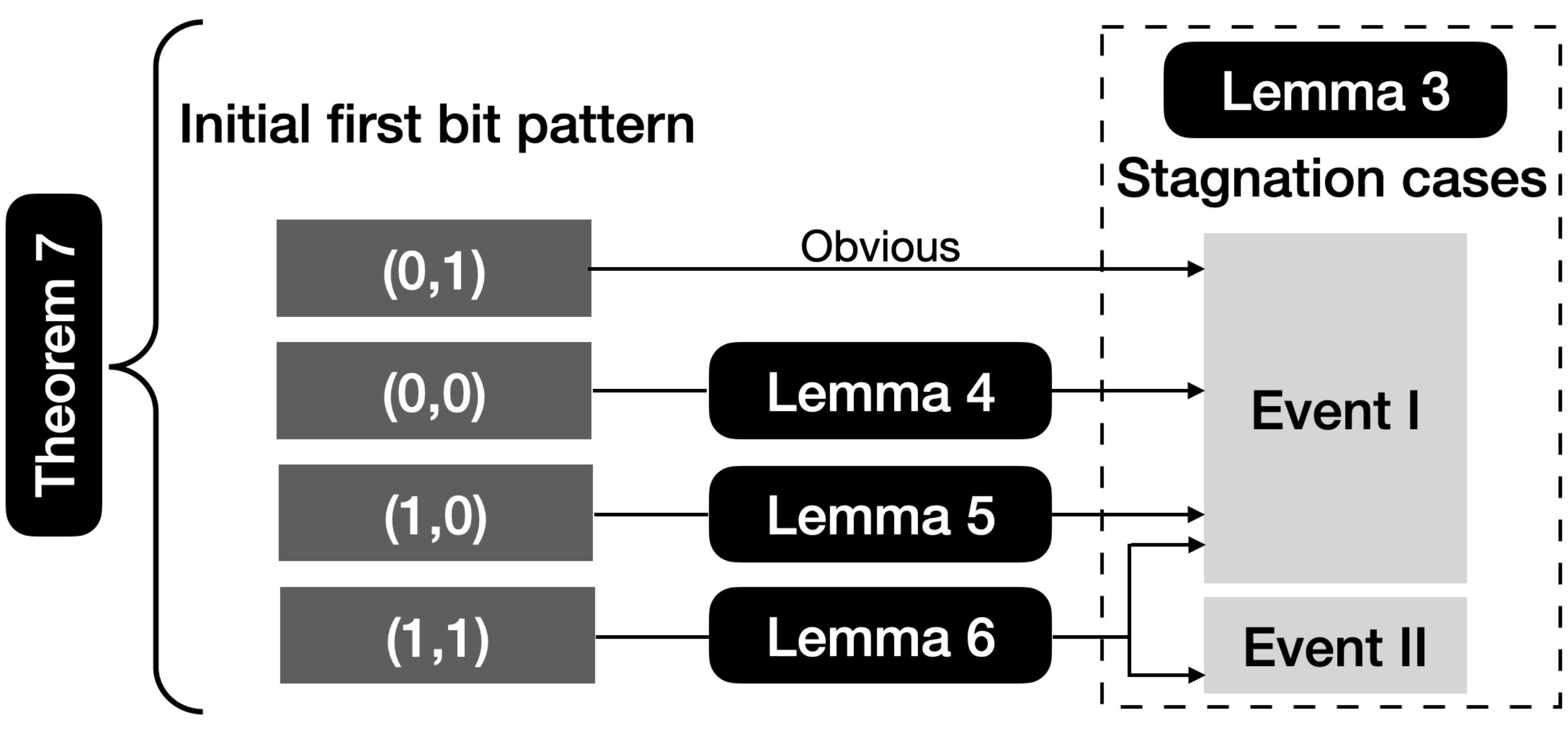}
\caption{Relationship between Theorem~\ref{thm:11EAOM} and Lemmas~\ref{lem:stuck},~\ref{lem:00},~\ref{lem:10}, and~\ref{lem:11}.}
\label{fig:relation}
\end{figure}

\begin{lemma}
Consider using \oea (RLS) to optimize $n$-dimensional OneMax$_{(0,1^n)}$ function. Let $X^0,X^1,\cdots$ be the solution sequence generated by the algorithm. Let 
\begin{itemize}
\item \emph{Event \textrm{I}}: There is a $g_0 \in \N$ such that $(X_{1}^{g_0-1},X_{1}^{g_0})=(0,1)$ and $X_{[2..n]}^{g_0} \neq 1^{n-1}$, 
\item \emph{Event \textrm{II}}: There is a $g_0 \in \N$ such that $(X_{1}^{g_0-1},X^{g_0})=(1,1^n)$. 
\end{itemize}
Then if at a certain generation among the solution sequence, \emph{Event \textrm{I}} or \emph{Event \textrm{II}} happens, then \oea (RLS) can not find the optimum of OneMax$_{(0,1^n)}$ in an arbitrary long runtime afterwards.
\label{lem:stuck}
\end{lemma}
\begin{proof}
Consider the case when \emph{Event \textrm{I}} happens, then $(X_{1}^{g_0-1},X_{1}^{g_0})=(0,1)$. In this case, the current fitness $f(X^{g_0-1},X^{g_0})\ge 1$. For every possible $\tilde{X}^{g_0}$, the mutation outcome of $X^{g_0}$, since $X_{1}^{g_0}=1$, we know $f(X^{g_0},\tilde{X}^{g_0})\le 0$ regardless of the values of other bits in $\tilde{X}^{g_0}$. Hence, $\tilde{X}^{g_0}$ cannot enter in the $(g_0+1)$-th generation, that is, any progress achieved in the OneMax component of OneMax$_{(0,1^n)}$ function (any bit value changing from $0$ to $1$ from the $2$-nd to the $n$-th bit position) cannot pass on to the next generation. Besides, $(X_{1}^{g_0},X_{1}^{g_0+1})=(X_{1}^{g_0-1},X_{1}^{g_0})=(0,1)$, then $(X_{1}^g,X_{1}^{g+1})=(0,1)$ holds for all $g\ge g_0-1$. That is, RLS or \oea gets stuck in this case. 

Consider the case when \emph{Event \textrm{II}} happens, that is, $(X_{1}^{g_0-1},X^{g_0})=(1,1^n)$. In this case, the current fitness $f(X^{g_0-1},X^{g_0})=0$. Similar to the above case, since $X_{1}^{g_0}=1$, all possible mutation outcome $\tilde{X}^{g_0}$ along with $X^{g_0}$ will have fitness less than or equal to $0$. Hence, only when $f(X^{g_0},\tilde{X}^{g_0})=0$ happens, $\tilde{X}^{g_0}$ can enter into the next generation, which means $\tilde{X}^{g_0}=1^n$. Therefore, RLS or \oea will get stuck in this case.
\end{proof}

\begin{lemma}
Consider using \oea (RLS) to optimize $n$-dimensional OneMax$_{(0,1^n)}$ function. Assume that at the first generation, $(X_{1}^0,X_{1}^1)=(0,0)$ and ${\sum_{j=2}^n X_{j}^1 < \frac 34 n}$. Then with probability at least $1-\frac{e+1}{n^{1/3}}$, \emph{Event \textrm{I}} will happen at one certain generation after this initial state.
\label{lem:00}
\end{lemma}
\begin{proof}
Consider the subsequent process once the number of $0$-bits among $\{2,\dots,n\}$ bit positions of the current individual, becomes less than $n^{c}$ for any given constant $c < 0.5$. Note that if the first bit value changes from $0$ to $1$ before the number of $0$-bits decreases to $n^{c}$, \emph{Event \textrm{I}} already happens. Hence, we just consider the case that the current first bit value is still $0$ at the first time when the number of the remaining $0$-bits decreases to $n^{c}$. Let $a$ denote the number of $0$-bits of the current individual. 
We conduct the proof of this lemma based on the following two facts.
\begin{itemize}
\item \emph{Among increase steps (the fitness has an absolute increase), a single increase step increases the fitness by $1$ with conditional probability at least $1-ea/n$}. For RLS, due to its one-bit mutation, the amount of fitness increase can only be $1$ for a single increase step. For \oea, Lemma~\ref{lem:improvefit} directly shows this fact.
\item \emph{Under the condition that one step increases the fitness by $1$, with conditional probability at least $1/a$, the first bit changes its value from $0$ to $1$}. It is obvious for RLS. For \oea, suppose that the number of bits changing from $0$ to $1$ in this step is $m\in[1..a]$, then the probability that the first bit contributes one $0$ is $\binom{a-1}{m-1}/\binom{a}{m}=m/a \ge 1/a$ for $m \ge 1$.
\end{itemize}
Note that there are $a-1$ increase steps before the remaining $n-1$ positions become all $1$s if each increase step increases the fitness by $1$. Then with the above two facts, it is easy to see that the probability that one individual with the $(0,1)$ first bit pattern is generated before remaining positions all have bit value $1$ is at least
\begin{align*}
\Bigg(&\prod_{a=n^{c}}^2\left(1-\frac{ea}{n}\right)\Bigg)\left(1-\prod_{a=n^{c}}^2\left(1-\frac{1}{a}\right)\right)\\
={}&\left(\prod_{a=n^{c}}^2 \left(1-\frac{ea}{n}\right)\right)\left(1-\prod_{a=n^{c}}^2\frac{a-1}{a}\right)\\
\ge{}& {\left(1-\frac{e}{n^{1-c}}\right)^{n^c}\left(1-\frac{1}{n^c}\right) \ge 1-\frac{e}{n^{1-2c}}-\frac{1}{n^c}}\\
\ge {}&  {1-\frac{e+1}{\min\{n^{1-2c},n^c\}}}. 
\end{align*}
We could just set $c=\frac13$ and obtain the probability lower bound as $1-\frac{e+1}{n^{1/3}}$.
\end{proof}

\begin{lemma}
Consider using \oea (RLS) to optimize $n$-dimensional OneMax$_{(0,1^n)}$ function. Assume that at the first generation, $(X_{1}^0,X_{1}^1)=(1,0)$ and ${\sum_{j=2}^n X_{j}^1 < \frac 34 n}$. Then with probability at least $1-\frac{e+1}{n^{1/3}}$, \emph{Event \textrm{I}} will happen at one certain generation after this initial state.
\label{lem:10}
\end{lemma}
\begin{proof}
Since $(X_{1}^0,X_{1}^1)=(1,0)$, we know that any offspring will have better fitness than the current individual, and will surely enter into the next generation. Then with probability $1/n$, the first bit value in the next generation becomes $1$, that is, \emph{Event \textrm{I}} happens. Otherwise, with probability $1-1/n$, it turns to the above discussed $(X_{1}^0,X_{1}^1)=(0,0)$ situation. Hence, in this situation, the probability that eventually \emph{Event \textrm{I}} happens is at lea st
\begin{align*}
& {\frac 1n+\left(1-\frac 1n\right)\left(1-\frac{e+1}{n^{1/3}}\right) \ge 1-\frac{e+1}{n^{1/3}}}.
\qedhere
\end{align*}
\end{proof}

\begin{lemma}
Consider using \oea (RLS) to optimize $n$-dimensional OneMax$_{(0,1^n)}$ function. Assume that at the first generation, $(X_{1}^0,X_{1}^1)=(1,1)$ and ${\sum_{j=2}^n X_{j}^1 < \frac 34 n}$. Then with probability at least $1-\frac{e+1}{n^{1/3}}-(n-1)\exp{\left(-\frac{n^{1/3}}{e}\right)}$, \emph{Event \textrm{I}} or \emph{Event \textrm{II}} will happen at one certain generation after this initial state.
\label{lem:11}
\end{lemma}
\begin{proof}
For $(X_{1}^0,X_{1}^1)=(1,1)$, since in each iteration only one bit can be flipped for RLS, once the first bit is flipped from $1$ to $0$, the fitness of the offspring will be less than its parent and the offspring cannot enter into the next generation. Hence, for RLS, the individual will be eventually evolved to $(X_{1}^{g_0-1},X^{g_0})=(1,1^n)$ for some $g_0\in\N$. That is, \emph{Event \textrm{II}} happens. 

For \oea, similar to the $(X_{1}^0,X_{1}^1)=(0,0)$ situation, we consider the subsequent process once the number of $0$-bits among $\{2,\dots,n\}$ bit positions of the current individual, becomes less than $n^{c}$ for some constant $c < 0.5$, and let $a$ denote the number of $0$-bits of the current individual. If the first bit value changes from $1$ to $0$ before the number of $0$-bits decreases to $n^{c}$, we turn to the $(X_{1}^0,X_{1}^1)=(1,0)$ situation. Otherwise, we will show that in the subsequent generations, with probability at least $1-o(1)$, the $(1,1)$ pattern will be maintained after the remaining bits reach the optimal $1^{n-1}$. 
Consider the condition that the first bit value stays at $1$s and let $\tilde{T}$ be the number of time that all $n-1$ bit positions have bit value $1$ under this condition. Note that under this condition, one certain $0$ bit position in $[2..n]$ does not change to $1$ in $t$ generations is at most $(1-(1-\frac1n)^{n-2}\frac1n)^t \le (1-\frac{1}{en})^t$. Then a union bound shows
\begin{align*}
\Pr&\left[\tilde{T} > n^{1+c} \mid \text{the $1$-st value stays at $1$}\right]  \\
\le{}& {(n-1) \left(1-\frac{1}{en}\right)^{n^{1+c}} \le (n-1)e^{-\frac{n^c}{e}}}.
\end{align*}
Noting that the probability that the offspring with the first bit value changing from $1$ to $0$ can enter into next generation is at most $\frac{1}{n}\frac{a}{n}=\frac{a}{n^2}\le\frac{1}{n^{2-c}}$, we can obtain the probability that \emph{Event \textrm{II}} happens within $n^{1+c}$ generations is at least
\begin{align*}
\bigg(1-&\frac{1}{n^{2-c}}\bigg)^{n^{1+c}} {\left(1-(n-1)e^{-\frac{n^c}{e}}\right)}\\
\ge{}& 1-\frac{n^{1+c}}{n^{2-c}}- {(n-1)e^{-\frac{n^c}{e}}} 
=1-\frac{1}{n^{1-2c}}- {(n-1)e^{-\frac{n^c}{e}}}.
\end{align*}
Further taking $c=\frac13$, together with the probability of \emph{Event I} when the first bit pattern changes to $(1,0)$ before the number of zeros decreases to $n^c$, we have \emph{Event \textrm{I}} or \emph{Event \textrm{II}} will happen with probability at least $1-\frac{e+1}{n^{1/3}}-(n-1)e^{-\frac{n^{1/3}}{e}}$.
\end{proof}

\begin{theorem}
Let $n\ge 6$. Then with probability at least $1-(n+1)\exp{\left(-\frac{n^{1/3}}{e}\right)}-\frac{e+1}{n^{1/3}}$, \oea (RLS) cannot find the optimum of the $n$-dimensional OneMax$_{(0,1^n)}$ function. 
\label{thm:11EAOM}
\end{theorem}
\begin{proof}
For the uniformly and randomly generated $1$-st generation, we have $E[\sum_{j=2}^n X_{j}^1] = (n-1)/2$. The Chernoff inequality in Lemma~\ref{lem:chernoff}(b) gives that ${\Pr[\sum_{j=2}^n X_{j}^1 \ge \tfrac 34 n] \le \exp(-(n-1)/8)}$. Hence, with probability at least $1-\exp(-(n-1)/8)$, $\sum_{j=2}^n X_{j}^1 < \tfrac 34n$ holds, that is, $[2..n]$ bit positions have at least $\frac14 n-1$ zeros. Thus, neither \emph{Event \textrm{II}} nor \emph{Event \textrm{II}} happens at the first generation.
We consider this initial status in the following.

Recall that from Lemma~\ref{lem:stuck}, \emph{Event \textrm{I}} and \emph{Event \textrm{II}} are two stagnation situations.
For the first bit values $X_{1}^0$ and $X_{1}^1$ of the randomly generated two individuals $X^0$ and $X^1$, there are four situations, $(X_{1}^0,X_{1}^1)=(0,1), (0,0), (1,0)$ or $(1,1)$. If $(X_{1}^0,X_{1}^1)=(0,1)$, \emph{Event \textrm{I}} already happens. Respectively from Lemma~\ref{lem:00}, Lemma~\ref{lem:10}, and Lemma~\ref{lem:11}, we could know that with probability at least $1-\frac{e+1}{n^{1/3}}-(n-1)e^{-\frac{n^c}{e}}$, \emph{Event \textrm{I}} or \emph{Event \textrm{II}} will happen in the certain generations after the initial first bit pattern $(0,0), (1,0)$ and $(1,1)$.

Overall, together with the probability of $\sum_{j=2}^n X_{j}^1 < \tfrac 34n$, we have the probability for getting stuck is at least 
\begin{align*}
\left(1-e^{-\frac{n-1}{8}}\right) {}&\left(1-\frac{e+1}{n^{1/3}}-(n-1)e^{-\frac{n^{1/3}}{e}}\right) \\
\ge{}& 1- e^{-\frac{n-1}{8}}-\frac{e+1}{n^{1/3}}-(n-1)e^{-\frac{n^{1/3}}{e}}\\
\ge {}& 1- (n+1)e^{-\frac{n^{1/3}}{e}}-\frac{e+1}{n^{1/3}}
\end{align*}
where the last inequality uses $\exp{(-\frac{n-1}{8})} \le 2\exp(-\frac{n}{8}) = 2\exp(-\frac{n^{1/3}}{e} \frac{en^{2/3}}{8}) \le 2\exp(-\frac{n^{1/3}}{e})$ when $n \ge 6$.
\end{proof}

One key reason causing the difficulty for (1+1) EA and RLS is that there is only one individual in the population. As we can see in the proof, once the algorithm finds the $(0,1)$ optimal pattern in the first bit, the progress in OneMax component cannot pass on to the next generation, and once the current OneMax component finds the optimum before the first bit $(0,1)$ optimal pattern, the optimal first bit pattern cannot be obtained further. In EAs, for some cases, the large population size does not help~\cite{ChenTCY12}, but the population could also have many benefits for ensuring the performance~\cite{HeY02,DangJL17,CorusO19,Sudholt20}. Similarly, we would like to know whether introducing population with not small size will improve the fault tolerance to overcome the first difficulty, and help to overcome the second difficulty since $(1,1^n)$ individual has worse fitness so that it is easy to be eliminated in the selection. The details will be shown in Section~\ref{sec:muoea}.

 {The above analyses are conducted for the offline mode. For the online mode on problem (\ref{eq:sumom}), as discussed in Sections~\ref{subsec:notes} and~\ref{subsec:oea}, we assume the time to make the decision that could change the fitness function value is so small that we can just expect one solution with a better or equivalent fitness value at each time step. Hence, when we reuse Algorithm~\ref{alg:1p1EA}, the time $t$ of the fitness function increases to $t+1$ once the algorithm witnesses the generation $g$ where the generated $\tilde{X}^g$ satisfies $f(X^g,\tilde{X}^g) \ge f(X^{g-1},X^g)$. It is not difficult to see that $t$ and $g$ are usually different as Algorithm~\ref{alg:1p1EA} may need more than one generation to generate such $\tilde{X}^g$. Hence, from Theorem~\ref{thm:11EAOM}, we could also obtain that with probability at least $1-(n+1)\exp{\left(-\frac{n^{1/3}}{e}\right)}-\frac{e+1}{n^{1/3}}$, \oea or RLS in our discussed online mode cannot find the solution at some time step $t$ when the function value of problem (\ref{eq:sumom}) is greater than $n-1$.}

\section{\mpoea Can Find the Optimum}
\label{sec:muoea}
\subsection{\mpoea Utilized for OneMax$_{(0,1^n)}$}
\mpoea is a commonly used benchmark algorithm for evolutionary theory analysis, which maintains a parent population of size $\mu$ comparing with \oea that has population size $1$. In the mutation operator, one parent is uniformly and randomly selected from the parent population, and the bit-wise mutation is employed on this parent individual and generates its offspring. Then the selection operator will uniformly remove one individual with the worse fitness value from the union individual set of the population and the offspring. 

Similar to \oea discussed in Section~\ref{sec:oea}, the general \mpoea is utilized for non-time-linkage function, and some small modifications are required for solving time-linkage problems. For solving OneMax$_{(0,1^n)}$ function in an offline mode, we just consider the representation that each individual in the population just encodes the current solution and stores the previous solution for the fitness evaluation. Similar to RLS and \oea, only offspring with better or equivalent fitness could affect the further fitness, hoping the optimization process to learn the preference of the time-linkage property. Algorithm~\ref{alg:mp1EA} shows how \mpoea solves the time-linkage function that relies on two consecutive time steps. With no confusion, we shall still call this algorithm \mpoea. Also note that we do not set the termination criterion in the algorithm statement, as we aim at theoretically analyzing the time to reach the optimum, that is, to generate $\tilde{X}^g=(1,\dots,1)$ condition on that its parent $X_{\tilde{i}}^g$ has the first bit value $0$ in Algorithm~\ref{alg:mp1EA}.

\begin{algorithm}[!ht]
    \caption{\mpoea to maximize fitness function $f$ requiring two consecutive time steps}
    {\small
    \begin{algorithmic}[1]
    \STATE {Generate the random initial two generations $P^0=\{X_1^0, \dots, X_{\mu}^0\}$ and $P^1=\{X_1^1, \dots, X_{\mu}^1\}$, where $X_i^{g}=(X_{i,1}^g,\dots, X_{i,n}^g), i=\{1,\dots, \mu\}, g=\{0,1\}$}
    \FOR {$g=1,2,\dots$}
    \STATEx {\quad$\%\%$ \textit{Mutation}}
    \STATE {Uniformly and randomly select one index $\tilde{i}$ from $[1..\mu]$}
    \STATE {Generate $\tilde{X}^g$ via independently flipping each bit value of $X_{\tilde{i}}^g$ with probability $1/n$}
    \STATEx {\quad$\%\%$ \textit{Selection}}
    \IF {$f(X_{\tilde{i}}^g,\tilde{X}^g) \ge  {\min\limits_{i\in\{1,\dots, \mu\}}\{f(X_i^{g-1},X_i^g)\}}$}
    \STATE {Let $(\tilde{P}^{g-1},\tilde{P}^g)=\{(P^{g-1},P^g),(X_{\tilde{i}}^g,\tilde{X}^g)\}$}
    \STATE {Remove the pair with the lowest fitness in $(\tilde{P}^{g-1},\tilde{P}^g)$ uniformly at random}
    \STATE {$P^{g+1}=\tilde{P}^g, P^g=\tilde{P}^{g-1}$}
    \ELSE
    \STATE {$P^{g+1}=P^g, P^g=P^{g-1}$}
    \ENDIF
    \ENDFOR
    \end{algorithmic}
    \label{alg:mp1EA}
    }
\end{algorithm}

We do not consider using \mpoea to solve OneMax$_{(0,1^n)}$ function in an online mode. If the decision must be made in a short time period as we discussed in Section~\ref{subsec:oea}, since different individuals in the parent population has their own evolving histories and different time fronts, the better offspring generated in one step cannot be regarded as the decision of the next step for the individuals in the parent population other than its own parent. If we have enough budget before time step changes, similar to the discussion in Section~\ref{subsec:oea}, we will have the fitness less than 1 for any time step since $X^t=(1,\dots,1)$ for each time step $t$. Also, it is not interesting for us. Hence, the following analysis only considers \mpoea (Algorithm~\ref{alg:mp1EA}) solving OneMax$_{(0,1^n)}$ function in the offline mode. 

\subsection{Convergence Analysis of \mpoea on OneMax$_{(0,1^n)}$}
\label{subsec:muoea}
In Section~\ref{sec:oea}, we show the two cases happening before the optimum is reached that cause the stagnation of \oea and RLS for the OneMax$_{(0,1^n)}$. One is that the $(0,1)$ first bit pattern is reached, and the other is that the current individual has the value one in all its bits with the previous first bit value as $1$. The single individual in the population of \oea or RLS results in the poor tolerance to the incorrect trial of the algorithm. This subsection will show that the introduction of population can increase the tolerance to the incorrect trial, and thus overcome the stagnation. That is, we will show that \mpoea can find the optimum of OneMax$_{(0,1^n)}$ with high probability. In order to give an intuitive feeling about the reason why the population can help for solving OneMax$_{(0,1^n)}$, we briefly and not-so-rigorously discuss it before establishing a rigorous analysis. 

Corresponding to two stagnation cases for \oea or RLS, \mpoea can get stuck when all individuals have the first bit value as $1$ no matter the previous first bit value $0$ as the first case or the previous first bit value $1$ as the second case. As discussed in Section~\ref{sec:oea}, we know that the individual with previous first bit value as $1$ has no fitness advantage against the one with previous first bit value as $0$. Due to the selection operator, the one with previous first bit value $1$ will be early replaced by the offspring with good fitness. As the process goes by, more detailedly in linear time of the population size in expectation, all individuals with the previous first bit value $1$ will die out, and the offspring with its parent first bit value $1$ cannot enter into the population. That is, the second case cannot take over the whole population to cause the stagnation.

As for the first case that the $(0,1)$ pattern individuals takes over the population, we focus on the evolving process of the best $(0,0)$ pattern individual, which is fertile, similar to the runtime analysis of original \mpoea in~\cite{Witt06}. The best $(0,0)$ pattern individuals can be incorrectly replaced only by the $(0,1)$ pattern individual with better or the same fitness and only when all individuals with worse fitness than the best $(0,0)$ pattern individual are replaced. With a sufficient large population size, like $\Omega(n)$ as $n$ the problem size, with high probability, the better $(0,1)$ pattern individuals cannot take over the whole population and the $(0,1)$ pattern individuals with the same fitness as the best $(0,0)$ pattern individual cannot replace all best $(0,0)$ individuals when the population doesn't have any individual with worse fitness than the best $(0,0)$ individual. That is, the first case with high probability will not happen for \mpoea. In a word, the population in \mpoea increases the tolerance to the incorrect trial.

Now we start our rigorous analysis. As we could infer from the above description, the difficulty of the theoretical analysis lies on the combining discussion of the inter-generation dependencies (the time-linkage among two generations) and the inner-generation dependencies (such as the selection operator). One way to handle these complicated stochastic dependencies could be the mean-field analysis, that is, mathematically analysis on a designed simplified algorithm that discards some dependencies and together with an experimental verification on the similarity between the simplified algorithm and the original one. It has been already introduced for the evolutionary computation theory~\cite{DoerrZ20}. However, the mean-field analysis is not totally mathematically rigorous. Hence, we don't utilize it here and analyze directly on the original algorithm. Maybe the mean-field analysis could help in more complicated algorithm and time-linkage problem, and we also hope our analysis could provide some other inspiration for the future theory work on the time-linkage problem.

For clarity, we put some calculations as lemmas in the following. 
\begin{lemma}
Let $a,n\in \N$, and $a<n$. Define the functions $h_1: [0,n-a-1] \rightarrow (0,1)$ and $h_2: [1,n-a] \rightarrow (0,1)$ by 
\begin{align*}
h_1(d)=\frac{\binom{a+d-1}{d}}{n^{d+1}}, h_2(d)=\frac{\binom{a+d-1}{d-1}}{n^{d}},
\end{align*}
then $h_1(d)$ and $h_2(d)$ are monotonically decreasing.
\label{lem:mono}
\end{lemma}
\begin{proof}
Since $h_1>0$, and for any $d_1,d_2 \in [0,n-a-1]$ and $d_1 \le d_2$,
\begin{align*}
\frac{h_1(d_1)}{h_1(d_2)}={}&\frac{\binom{a+d_1-1}{d_1}}{n^{d_1+1}}\frac{n^{d_2+1}}{\binom{a+d_2-1}{d_2}}\\
={}&n^{d_2-d_1}\frac{(a+d_1-1)!}{(a-1)!d_1!}\frac{(a-1)!d_2!}{(a+d_2-1)!}\\
={}&n^{d_2-d_1}\frac{d_2 \cdots (d_1+1)}{(a+d_2-1)\cdots(a+d_1)}\\
\ge{}& n^{d_2-d_1}\left(\frac{d_1+1}{a+d_1}\right)^{d_2-d_1} 
\ge\left(\frac{n}{n-1}\right)^{d_2-d_1} \ge 1,
\end{align*}
we know $h_1(d)$ is monotonically decreasing.

Similarly, since $h_1>0$, and for any $d_1,d_2 \in [1,n-a]$ and $d_1 \le d_2$,
\begin{align*}
\frac{h_2(d_1)}{h_2(d_2)}={}&\frac{\binom{a+d_1-1}{d_1-1}}{n^{d_1}}\frac{n^{d_2}}{\binom{a+d_2-1}{d_2-1}}\\
={}&n^{d_2-d_1}\frac{(a+d_1-1)!}{a!(d_1-1)!}\frac{a!(d_2-1)!}{(a+d_2-1)!}\\
={}&n^{d_2-d_1}\frac{(d_2-1) \cdots d_1}{(a+d_2-1)\cdots(a+d_1)}\\
\ge{}& n^{d_2-d_1}\left(\frac{d_1}{a+d_1}\right)^{d_2-d_1} \ge n^{d_2-d_1}\left(\frac{1}{n}\right)^{d_2-d_1} = 1,
\end{align*}
we know $h_2(d)$ is monotonically decreasing.
\end{proof}

\begin{lemma}
Let $n\in \N$. Define the function $g: [1,n^{1/2}] \rightarrow (0,1)$ by $g(a)=\frac{a^a}{n^{a^2}}$, then $g(a)$ is monotonically decreasing.
\label{lem:mono2}
\end{lemma}
\begin{proof}
Consider $\tilde{g}(a)=\ln g(a) = a\ln a - a^2\ln n$. For any $a_1, a_2 \in [1, n^{1/2}]$ and $a_1 < a_2$, we have
\begin{align*}
\tilde{g}(a_1) -{}& \tilde{g}(a_2) = a_1\ln a_1 - a_1^2\ln n - a_2\ln a_2 +a_2^2\ln n \\
\ge{}& (a_1+a_2)\ln n - a_2\ln a_2 > 0.
\end{align*}
Then, $\tilde{g}(a)$, and hence $g(a)$, are monotonically decreasing.
\end{proof}

\begin{lemma}
Let $n > (4e)^2$. Then $\left(\frac{3}{4}\right)^{n^{1/2}-1} \le n^{-1/2}$.
\label{lem:cal}
\end{lemma}
\begin{proof}
Consider $s(n)=-\left(n^{1/2}-1\right)\frac14 + \frac12\ln n$. Then for $n>(4e)^2 > 16$,
\begin{align*}
s'(n)=-\tfrac14 \cdot \tfrac12 n^{-1/2}+\tfrac12n^{-1}=\tfrac 18 n^{-1} (-n^{1/2}+4) <0.
\end{align*}
Hence, $s(n)$ is monotonically decreasing for $n>(4e)^2$. Since $s((4e)^2) =-\frac{4e-1}{4}+\ln (4e) <0$, we have ${s(n) < 0}$ for $n>(4e)^2$. Noting that $-\frac 14 > \ln \frac{3}{4}$, we have $\left(n^{1/2}-1\right)\ln \frac{3}{4} < -\frac12\ln n$ and hence $\left(\frac{3}{4}\right)^{n^{1/2}-1} \le n^{-1/2}$ for $n > (4e)^2$.
\end{proof}

Now we show our results that \mpoea can find the optimum of OneMax$_{(0,1^n)}$ function with high probability. We start with some definitions that will be frequently used in our proofs.


\begin{definition}
Consider using \mpoea with population size $\mu$ to optimize $n$-dimensional OneMax$_{(0,1^n)}$ function. $P^g$ is the population at the $g$-th generation, and $P^{g-1}$ for the $(g-1)$-th generation. Let $l=\max\{f(X_i^{g-1},X_i^g) \mid (X_{i,1}^{g-1},X_{i,1}^g) = (0,0), X_{i}^{g-1} \in P^{g-1}, X_{i}^g\in P^g, i\in[1..\mu]\}$ be the best fitness among all individuals with the $(0,0)$ first bit pattern. For $X_i^g \in P^g, i \in [1..\mu]$,
\begin{itemize}
\item it is called a \emph{temporarily undefeated individual} if $f(X_i^{g-1},X_i^g) > l$ and $(X_{i,1}^{g-1},X_{i,1}^g)=(0,1)$.
\item it is called a \emph{current front individual} if $f(X_i^{g-1},X_i^g) = l$ and $(X_{i,1}^{g-1},X_{i,1}^g)=(0,0)$.
\item it is called an \emph{interior individual} if it is neither a temporarily undefeated individual nor a current front individual.
\end{itemize}
\end{definition}

The most difficult case is when all $(1,0)$ and $(1,1)$ are replaced, which happens after linear time of the population size as in the proof of Theorem~\ref{thm:mpoEAOM}. Lemmas~\ref{lem:acctui} to \ref{lem:occupied} discuss the behavior in this case. Lemma~\ref{lem:acctui} will show that when the population is large enough, with high probability, there are at most $\frac12 \mu-1$ accumulative temporarily undefeated individuals before the current front individuals have only 1 zero.
\begin{lemma}
Given any $\delta > 0$. For $n > (4(1+\delta)e)^2$, consider using \mpoea with population size $\mu \ge 4(1+\delta)(3e+1)(n+1)$ to optimize $n$-dimensional OneMax$_{(0,1^n)}$ function. Assume that
\begin{itemize}
\item the current front individuals have less than $\frac34 n$ zeros;
\item all individuals are with the $(0,0)$ or $(0,1)$ pattern.
\end{itemize} 
Then with probability at least $1-\exp\left(-\frac{\delta^2}{2(1+\delta)}(n-1)\right)$, there are at most $\frac12 \mu-1$ accumulative temporarily undefeated individuals generated before the current front individuals only have 1 zero.
\label{lem:acctui}
\end{lemma}
\begin{proof}
Consider the case when the current front individuals have at least 2 zeros. For the current population, let $a$ denote the number of zeros in one current front individual, then $a\ge 2$. Let $m_d$ denote the set of the $(0,0)$ individuals that have $a+d$ number of zeros. Obviously, $m_0$ is the set of the best $(0,0)$ individuals. Let $A$ represent the event that the best $(0,0)$ fitness of the population increases in one generation, and $B$ the event that one $(0,1)$ offspring with better fitness than the current best $(0,0)$ fitness is generated in one generation. Then we have
\begin{align*}
\Pr[A]\ge{}&\sum_{d \ge 0}\frac{|m_d|}{\mu}\frac{\binom{a+d-1}{d+1}}{n^{d+1}}\left(1-\frac{1}{n}\right)^{n-d-1}\\
\ge{}& \sum_{d \ge 0}\frac{|m_d|}{e\mu}\frac{\binom{a+d-1}{d+1}}{n^{d+1}}
\end{align*}
and
\begin{align*}
\Pr[B] \le \sum_{d \ge 0}\frac{|m_d|}{\mu}\frac{\binom{a+d-1}{d}}{n^{d+1}}.
\end{align*}

Firstly, we discuss what happens under the condition that the parent is not from $m_{> a}$ individuals. Let $B'$ represent the event that one of $m_{> a}$ individuals generates a $(0,1)$ offspring with better fitness than the current best $(0,0)$ fitness in one generation. Since
\begin{align*}
\frac{\binom{a+d-1}{d+1}}{\binom{a+d-1}{d}} = \frac{(a+d-1)!}{(d+1)!(a-2)!} \cdot \frac{d!(a-1)!}{(a+d-1)!} = \frac{a-1}{d+1},
\end{align*}
we know
\begin{equation}
\begin{split}
\frac{\Pr[A]}{\Pr[B-B']} \ge{}& \frac{\sum_{d = 0}^{a}\frac{ {|m_d|}}{e\mu}\frac{\binom{a+d-1}{d+1}}{n^{d+1}}}{\sum_{d = 0}^{a}\frac{ {|m_d|}}{\mu}\frac{\binom{a+d-1}{d}}{n^{d+1}}} \ge \frac{a-1}{e(a+1)} \\
={}& \frac 1e \left(1-\frac{2}{a+1}\right) \ge \frac{1}{3e}
\label{eq:abbp}
\end{split}
\end{equation}
where the last inequality uses $a\ge 2$.

Secondly, we consider the case when the parent is selected from $m_{> a}$ individuals, that it, event $B'$ happens. 
With Lemma~\ref{lem:mono}, we know
\begin{align*}
\Pr[B'] \le{}& \sum_{d > a}\frac{ {|m_d|}}{\mu}\frac{\binom{a+d-1}{d}}{n^{d+1}} \le \sum_{d > a}\frac{ {|m_d|}}{\mu}\frac{\binom{2a-1}{a}}{n^{a+1}}\le\frac{\binom{2a-1}{a}}{n^{a+1}} \\
={}& \frac{(2a-1)!}{a!(a-1)!n^{a+1}} 
\le \frac{(a+1)^{a-1}}{n^{a+1}}=\frac{1}{n^2}\left(\frac{a+1}{n}\right)^{a-1}.
\end{align*} 
We discuss in two cases considering $a \ge n^{c}$ and $a < n^{c}$ for any given constant $c\in (0,1)$. From the assumption in this lemma, we know $a < \tfrac 34 n$. Then for $a \ge n^{c}$, we have 
\begin{align*}
\left(\frac{a+1}{n}\right)^{a-1} \le \left(\frac 34\right)^{a-1} \le \left(\frac 34\right)^{n^{c}-1}.
\end{align*}
For $a \in [2, n^{c})$,
\begin{align*}
\left(\frac{a+1}{n}\right)^{a-1} \le \left(\frac{n^{c}}{n}\right)^{a-1} \le \frac{1}{n^{1-c}}.
\end{align*}
Hence, we have for $a \in [2,\tfrac34 n)$, $\Pr[B']\le n^{c-3}$ since $(\tfrac{3}{4})^{n^{c}-1} \le n^{c-1}$ when $n$ is large. Together with $\Pr[A]\ge 1/(e\mu n)$, we know
\begin{align}
\frac{\Pr[A]}{\Pr[B']} \ge \frac{n^{2-c}}{e\mu}.
\label{eq:abp}
\end{align}

Hence, from (\ref{eq:abbp}) and (\ref{eq:abp}), we obtain
\begin{align*}
\frac{\Pr[A]}{\Pr[B]} \ge \frac{1}{\frac{e\mu}{n^{2-c}}+3e}.
\end{align*}
Then if we consider the subprocess that merely consists of event $A$ and $B$, we have $\Pr[A \mid A \cup B] \ge 1/(\frac{e\mu}{n^{2-c}}+3e+1)$. Let $X$ be the number of iterations that $B$ happens in the subprocess before $A$ occurs $n$ times, then $E[X]\le (\frac{e\mu}{n^{2-c}}+3e+1)n$. It is not difficult to see that $X$ is stochastically dominated by the sum of $n$ geometric variables with success probability $\frac{1}{\frac{e\mu}{n^{2-c}}+3e+1}$. Hence, with
 the Chernoff bound for the sum of geometric variables in Lemma~\ref{lem:chernoff}(c), we have that for any positive constant $\delta$,
\begin{align*}
\Pr{}&\left[X\ge(1+\delta)\left(\frac{e\mu}{n^{2-c}}+3e+1\right)n\right] 
\le \exp\left(-\frac{\delta^2(n-1)}{2(1+\delta)}\right).
\end{align*}

When $n > (4(1+\delta)e)^{\frac{1}{1-c}}$ and $\mu \ge 4(1+\delta)(3e+1)(n+1)$, we know
\begin{align*}
(1+\delta)\left(\frac{e\mu}{n^{2-c}}+3e+1\right)n 
={}&  \frac{(1+\delta)e\mu}{n^{1-c}} + (1+\delta) (3e+1)n \\
<{}& \tfrac{\mu}{4} + \tfrac{\mu}4-1 = \tfrac 12 \mu-1.
\end{align*}
Hence, we know that with probability at least $1-\exp\left(-\frac{\delta^2}{2(1+\delta)}(n-1)\right)$, there are at most $\frac 12 \mu -1$ possible accumulative temporarily undefeated individuals before $A$ occurs $n$ times. Now take $c=\frac12$, then $(4(1+\delta)e)^{\frac{1}{1-c}}=(4(1+\delta)e)^2$. With $n > (4(1+\delta)e)^2 > (4e)^2$, from Lemma~\ref{lem:cal}, we know $(\tfrac{3}{4})^{n^{c}-1} \le n^{c-1}$ and thus (\ref{eq:abp}) hold. Noting that reducing the number of zeros in one current front individuals to 1 requires at most $n-1$ occurrence times of $A$, this lemma is proved.
\end{proof}

Lemma~\ref{lem:occupy} will show that with high probability, the current front individuals will get accumulated to $\Theta(n^{0.5})$ before all interior individuals have the same number of zeros as the current front individuals.
\begin{lemma}
Let $n>(4e)^2$. Consider using \mpoea with population size $\mu$ to optimize $n$-dimensional OneMax$_{(0,1^n)}$ function. 
\begin{itemize}
\item Assume that 
\begin{itemize}
\item there are at most $\frac12 \mu-1$ accumulative temporarily undefeated individuals before the current front individuals only have 1 zero;
\item all individuals are with the $(0,0)$ or $(0,1)$ pattern;
\end{itemize} 
\item Consider the phase starting from the first time when the current front individuals have $a \in [1..n]$ zeros, and ending with the first time when one $(0,0)$ pattern individual with less than $a$ zeros is generated for $a\in[2..n]$, or ending with the first time when one $(0,1)$ pattern individual with all ones is generated for $a=1$. 
\end{itemize}
Then with probability at least $1-\exp\left(-\tfrac{1}{20}n^{0.5}\right)$, the current front individuals will increase by more than $\tfrac{1}{5}n^{0.5}$ if the current phase doesn't end before all individuals have at most $a$ zeros.
\label{lem:occupy}
\end{lemma}
\begin{proof}
Same as the notation in the proof of Lemma~\ref{lem:acctui}, $m_d, d\ge 0$ denotes the set of the $(0,0)$ individuals that have $a+d$ number of zeros. We now analyze the change of $|m_0|$ until all other individuals have at most $a$ zeros, if possible, during the phase. Let $C$ represent the event that one $(0,0)$ pattern individual with $a$ zeros is generated in one generation, and $D$ the event that one $(0,1)$ pattern individual with $a$ zeros is generated in one generation. Then we have
\begin{align*}
\Pr[C]\ge{}&\sum_{d \ge 0}\frac{|m_d|}{\mu}\frac{\binom{a+d-1}{d}}{n^{d}}\left(1-\frac{1}{n}\right)^{n-d}\\
\ge{}& \sum_{d \ge 1}\frac{|m_d|}{e\mu}\frac{\binom{a+d-1}{d}}{n^{d}}+\frac{|m_0|}{\mu}\left(1-\frac 1n\right)^n
\end{align*}
and
\begin{align*}
\Pr[D] \le \sum_{d \ge 1}\frac{|m_d|}{\mu}\frac{\binom{a+d-1}{d-1}}{n^{d}}+\frac{|m_0|}{\mu}\frac{\binom{a-1}{1}}{n^2}
\end{align*}
where we define $\binom{0}{1}=1$ for $a=1$. 
Assume that the parent is not from $m_{>a^2}$ individuals. Let $D'$ represent the event that one of $m_{> a^2}$ individuals generates a $(0,1)$ offspring with $a$ zeros in one generation. Due to the definition of $m_{d}$, when $m_{>a^2}$ exist, we have $a+a^2 \le n$, and then $a<n^{0.5}$. Since
\begin{align*}
\frac{\binom{a+d-1}{d}}{\binom{a+d-1}{d-1}} = \frac{(a+d-1)!}{d!(a-1)!} \cdot \frac{(d-1)!a!}{(a+d-1)!} = \frac{a}{d},
\end{align*}
we know
\begin{equation}
\begin{split}
\frac{\Pr[C]}{\Pr[D-D']} \ge{}& \frac{\sum_{d = 1}^{a^2}\frac{|m_d|}{e\mu}\frac{\binom{a+d-1}{d}}{n^{d}}+\frac{|m_0|}{\mu}\left(1-\frac 1n\right)^n}{\sum_{d = 1}^{a^2}\frac{|m_d|}{\mu}\frac{\binom{a+d-1}{d-1}}{n^{d}}+\frac{|m_0|}{\mu}\frac{\binom{a-1}{1}}{n^2}}\\
\ge{}& \frac{a}{e a^2} = \frac{1}{ea} > \frac{1}{en^{0.5}}.
\label{eq:cddp}
\end{split}
\end{equation}

Now we calculate $\Pr[D']$, the probability that one of $m_{> a^2}$ individual generates a $(0,1)$ offspring with $a$ zeros in one generation, as
\begin{equation}
\begin{split}
\Pr[D'] \le{}& \sum_{d > a^2}\frac{ {|m_d|}}{\mu}\frac{\binom{a+d-1}{d-1}}{n^{d}} \le \sum_{d > a^2}\frac{ {|m_d|}}{\mu}\frac{\binom{a+a^2-1}{a^2-1}}{n^{a^2}} \\
\le{}& \frac{\binom{a+a^2-1}{a^2-1}}{n^{a^2}} = \frac{(a+a^2-1)!}{(a^2-1)!a! n^{a^2}} 
\le \frac{a^a}{n^{a^2}} \le  {\frac{1}{n}},
\label{eq:d'}
\end{split}
\end{equation}
where the second inequality follows from Lemma~\ref{lem:mono} and the last inequality follows from Lemma~\ref{lem:mono2} and $a \ge 1$. With $\Pr[C] \ge (1-1/n)^n/\mu \ge 1/(2e \mu)$ for $n\ge 2$, we have
\begin{align}
\frac{\Pr[C]}{\Pr[D']} \ge \frac{n}{2e\mu}.
\label{eq:cdp}
\end{align}

Hence, from (\ref{eq:cddp}) and (\ref{eq:cdp}), we obtain
\begin{align*}
\frac{\Pr[C]}{\Pr[D]} \ge  {\frac{1}{\frac{2e\mu}{n}+en^{0.5}}}.
\end{align*}
Then if we consider the subprocess that merely consists of event $C$ and $D$, we have $\Pr[C \mid C \cup D] \ge 1/(\frac{2e\mu}{n}+en^{0.5}+1)$. Recalling the definition of the phase we consider, it is not difficult to see that at the initial generation of this phase, there is only one $(0,0)$ front individual with $a$ number of zeros, and not difficult to see that all $(0,1)$ individuals with $a$ zeros, if exist, are temporarily undefeated individuals of the previous phase. Noting the assumption that there are at most $\frac 12 \mu -1$ accumulative temporarily undefeated individuals before $a=1$, we know there are at least $\frac 12 \mu$ individuals that have more than $a$ zeros for the current phase. Hence, it requires at least $\frac 12 \mu$ number of steps of the subprocess to replace these individuals with more than $a$ zeros. Let $Y$ be the number of times that $C$ happens in $\frac 12 \mu$ steps of the subprocess, then 
\begin{align*}
E[Y] \ge{}& \frac{\frac12 \mu}{\frac{2e\mu}{n}+en^{0.5}+1}\\
\ge{}& \min\left\{\frac{\frac12 \cdot 4(1+\delta)(3e+1)(n+1)}{2en^{0.5}+1}, \frac{\frac12\mu}{2\frac{2e\mu}{n} +1} \right\}\\
\ge{}& \min\{\tfrac{6en}{3en^{0.5}},\tfrac{1}{10e}n\} \ge \tfrac{2}{5}n^{0.5}
\end{align*}
where the last inequality uses $\frac{1}{10e}n \ge \frac{1}{10e}4en^{0.5}=\frac{2}{5}n^{0.5}$ for $n\ge (4e)^2$.
The Chernoff inequality in Lemma~\ref{lem:chernoff}(a) gives that $\Pr[Y\le \frac{1}{5}n^{0.5}] \le \exp\left(-\tfrac{1}{20}n^{0.5}\right)$. That is, with probability at least $1-\exp\left(-\tfrac{1}{20}n^{0.5}\right)$, $|m_0|$ will increase by more than ${\tfrac{1}{5}n^{0.5}}$ if current phase does not end before all individuals have at most $a$ zeros.  
\end{proof}

Lemma~\ref{lem:occupied} shows that with high probability, it cannot happen that all current front individuals are replaced by the $(0,1)$ pattern individuals.
\begin{lemma}
Let $n>(4e)^2$. Consider using \mpoea with population size $\mu$ to optimize $n$-dimensional OneMax$_{(0,1^n)}$ function. Considered the same assumptions and phase as in Lemma~\ref{lem:occupy}. Further assume that at one generation of the current phase, there are $|m_0| > \tfrac{1}{5}n^{0.5}$ current front individuals with $a$ zeros, and all interior individuals has $a$ zeros. 
Then with probability at least $1-\big(\frac{e}{e+1}\big)^{n^{0.5}/5}$, the current phase ends before $|m_0|$ decreases to 0.
\label{lem:occupied}
\end{lemma}
\begin{proof}
Let $F$ denote the event that the current phase ends, that is, a $(0,0)$ offspring with at most $a-1$ zeros when $a\ge2$ or a $(0,1)$ offspring with 0 zero when $a=1$ is generated. Then 
\begin{align*}
\Pr[F\mid a \ge 2] &\ge \tfrac{|m_0|}{e\mu}\tfrac{a-1}{n} \ge \tfrac{|m_0|}{en\mu} \\
\Pr[F\mid a = 1] &\ge \tfrac{|m_0|}{e\mu}\tfrac{1}{n} = \tfrac{|m_0|}{en\mu}.
\end{align*}
Then $\Pr[F]\ge \tfrac{|m_0|}{en\mu}$. 

Let $G$ denote the event that a $(0,1)$ offspring with $a$ zeros is generated and one $(0,0)$ individual is replaced. Suppose that the total number of the individuals with $a$ zeros is $\mu'$.
Then
\begin{align*}
\Pr[G] \le \frac{|m_0|}{\mu}\frac 1n \frac{|m_0|}{ {\mu'+1}}\le\frac{|m_0|^2}{n\mu ( {\mu'+1})}.
\end{align*} 
Hence,
\begin{align*}
\frac{\Pr[F]}{\Pr[G]} \ge \frac{|m_0|}{en\mu} \cdot \frac{n\mu ( {\mu'+1})}{|m_0|^2} = \frac{ {\mu'+1}}{ {e}|m_0|} \ge  {\frac {1}{e}}
\end{align*}
where the last inequality uses $\mu' \ge |m_0|$. Then
\[
\Pr[G \mid F \cup G] \le \frac{e}{e+1}.
\]
Then the probability that $G$ happens $|m_0| > \tfrac{1}{5}n^{0.5}$ times but $F$ does not happen is at most $\big(\frac{e}{e+1}\big)^{n^{0.5}/5}$, and this lemma is proved.
\end{proof}

Now the main result follows.
\begin{theorem}
Given any $\delta > 0$. Let $n > (4(1+\delta)e)^2$. Then with probability at least $1-(\mu+2) \exp(-\tfrac18 n) -\exp\left(-\frac{\delta^2}{2(1+\delta)}(n-1)\right)-2n\exp(-\tfrac{1}{20}n^{0.5})$, \mpoea with population size $\mu \ge 4(1+\delta)(3e+1)(n+1)$ can find the optimum of OneMax$_{(0,1^n)}$.
\label{thm:mpoEAOM}
\end{theorem}
\begin{proof}
Similar to two situations in Lemma~\ref{lem:stuck} that could possibly result in the stagnation of \oea and RLS, the only two possible cases that could result in the stagnation of \mpoea are listed in the following.
\begin{itemize}
\item \emph{Event \textrm{I}'}: There is a $g_0 \in \N$ such that for all $i \in [1..\mu]$, $(X_{i,1}^{g-1},X_{i,1}^g)=(0,1)$ and $X_{i,[2..n]}\neq 1^{n-1}$.
\item \emph{Event \textrm{II}'}: There is a $g_0 \in \N$ such that for all $i \in [1..\mu]$, $(X_{i,1}^{g-1},X_{i}^g)=(1,1^n)$.
\end{itemize}


For the uniformly and randomly generated $P^0$ and $P^1$, we know that the expected number of the $(1,0)$ or $(1,1)$ first bit patterns, that is, the expected cardinality of the set $\{i\in[1..\mu] \mid (X_{i,1}^0,X_{i,1}^1)=(1,0)$ or $(1,1)\}$, is $\tfrac{\mu}{2}$. Via the Chernoff inequality in Lemma~\ref{lem:chernoff}(b), we know that with probability at most $1-\exp(-\tfrac{\mu}{8})$, at most $\tfrac 34 \mu$ individuals have the pattern $(1,0)$ or $(1,1)$. Under this condition, the expected number of the $(0,0)$ pattern is at least $\tfrac 18 \mu$ in the whole population. The Chernoff inequality in Lemma~\ref{lem:chernoff}(a) also gives that under the condition that at most $\tfrac 34 \mu$ individuals have the pattern $(1,0)$ or $(1,1)$, with probability at least $1-\exp(-\tfrac{\mu}{64})$, there are at least $\tfrac{1}{16}\mu$ individuals with the pattern $(0,0)$ for the initial population. Since $E[\sum_{j=1}^n(1-X_{i,j}^1)] = \tfrac12 n$ for any $i\in[1..\mu]$, the Chernoff inequality in Lemma~\ref{lem:chernoff}(b) on the initial population gives that $\Pr[\sum_{j=1}^n(1-X_{i,j}^1) \ge \frac 34 n] \le \exp(-\tfrac18 n)$. Via a union bound, we know 
\begin{align*}
\Pr\left[\exists i_0 \in [1..\mu], \sum\nolimits_{j=1}^n\left(1-X_{i_0,j}^1\right) \ge \frac 34 n\right] \le  \mu \exp(-\tfrac18 n).
\end{align*}
Hence, noting $\tfrac{\mu}{64} \ge \frac{12e}{64}n > \frac18 n$ from $\mu \ge 4(1+\delta)(3e+1)(n+1)$, it is easy to see that with probability at least 
\begin{align*}
1-&\exp(-\tfrac{\mu}{8})-\exp(-\tfrac{\mu}{64})-\mu \exp(-\tfrac18 n)\\
\ge{}& 1-2\exp(-\tfrac{\mu}{64})-\mu \exp(-\tfrac18 n) \ge 1-(\mu+2)\exp(-\tfrac18 n),
\end{align*} 
the initial population has at most $\tfrac 34 \mu$ individuals with the pattern $(1,0)$ or $(1,1)$, at least $\tfrac{1}{16}\mu$ individuals with the pattern $(0,0)$, and $a<\frac 34 n$ at the first generation. Thus, in the following, we just consider this kind of initial population.

We first show that after $O(\mu)$ generations, the individuals with the first bit pattern $(1,0)$ or $(1,1)$ will be replaced and will not survive in any further generation, thus \emph{Event \textrm{II}'} cannot happen. Since $a<\frac34 n$, we know the current front individual has fitness more than $1$. Note that all $(1,0)$ or $(1,1)$ pattern individuals have fitness at most ${0}$. Then any offspring copied from one $(0,0)$ pattern individual, which has the $(0,0)$ pattern and the same fitness as its parent, will surely enter into the generation and replace some individual with the $(1,0)$ or $(1,1)$ pattern. Since there is at least $\tfrac{1}{16}\mu$ individuals with the pattern $(0,0)$, we know that for each generation, with probability at least $\tfrac{1}{16}$, one individual with the $(1,0)$ or $(1,1)$ pattern will be replaced. Hence, the expected time to replace all individuals with the pattern $(1,0)$ or $(1,1)$ is at most $16\cdot \tfrac 34 \mu=12\mu$ since there are at most $\tfrac 34 \mu$ individuals with the pattern $(1,0)$ or $(1,1)$. Also it is not difficult to see any offspring with the $(1,0)$ or $(1,1)$ pattern cannot be selected into the next generation for a population only with the $(0,0)$ or $(0,1)$ pattern. Later, only the $(0,1)$ and $(0,0)$ first bit pattern can survive in the further evolution. 

Now we consider \emph{Event I'} after the first time when all $(1,0)$ or $(1,1)$ pattern individuals are replaced.
From Lemma~\ref{lem:acctui}, we know that before $a=1$, with probability at least $1-\exp\left(-\frac{\delta^2}{2(1+\delta)}(n-1)\right)$, there are at most $\frac12 \mu-1$ accumulative temporarily undefeated individuals. From Lemma~\ref{lem:occupy} among all possible $a\in[1..n]$, we know that if the optimum is not found before all individuals with at least $2$ zeros are replaced, with probability at least $1-n\exp\left(-\tfrac{1}{20}(n^{0.5}-1)\right)$, there are at least $\tfrac{1}{5}n^{0.5}$ number of $(0,0)$ individuals with $a=1$. Then for the case $a=1$ in Lemma~\ref{lem:occupied} and considering all possible $a\in[1..n]$, we know with probability at least $1-n\big(\frac{e}{e+1}\big)^{n^{0.5}/5}$, the optimum is found.

Overall, the probability that \mpoea can find the optimum of OneMax$_{(0,1^n)}$ is at least
\begin{align*}
\bigg(1-&(\mu+2) e^{-\tfrac18 n}\bigg)\left(1-\exp\left(-\frac{\delta^2}{2(1+\delta)}(n-1)\right)\right)\\
&\cdot \Big(1-n\exp\left(-\tfrac{1}{20}n^{0.5}\right)\Big)\left(1-n\left(\frac{e}{e+1}\right)^{\tfrac{1}{5}n^{0.5}}\right)\\
\ge{}& 1-(\mu+2) e^{-\tfrac18 n} -\exp\left(-\frac{\delta^2}{2(1+\delta)}(n-1)\right)  \\
{}&-n\exp\left(-\tfrac{1}{20}n^{0.5}\right)- n\left(\frac{e}{e+1}\right)^{\tfrac{1}{5}n^{0.5}}\\
\ge{}& 1-(\mu+2) e^{-\tfrac18 n} -\exp\left(-\frac{\delta^2}{2(1+\delta)}(n-1)\right)  \\
{}&-2n\exp\left(-\tfrac{1}{20}n^{0.5}\right)
\end{align*}
where the last inequality uses 
$e^{-1/4} > \frac{e}{e+1}$.
\end{proof}
In Theorem~\ref{thm:mpoEAOM}, we require that the population size $\mu \ge 4(1+\delta)(3e+1)(n+1)$. One may ask about the behavior when 
$\mu=o(n)$. 
We note in the proof of Lemma~\ref{lem:acctui}, the upper bound for the expected number of accumulative temporarily undefeated individuals is $(\frac{e\mu}{n^{2-c}}+3e+1)n=\Omega(n)$. If $\mu=o(n)$, we are not able to ensure that the accumulative temporarily undefeated individuals do not take over the population in our current proof, hence, we require $\mu=\Omega(n)$.

Comparing with $(1+1)$~EA, since $(1,1^n)$ individual, corresponding to \emph{Event \textrm{II}} in $(1+1)$~EA, has no fitness advantage against the one with previous first bit value as $0$, it is easy to be replaced by the offspring with previous first bit value $0$ in a population. Thus, this stagnation case cannot take over the whole population to cause the stagnation of $(\mu+1)$~EA. The possible stagnation case that the $(0,1)$ pattern individuals take over the population, corresponding to \emph{Event \textrm{I}} in $(1+1)$~EA, will not happen with a high probability because with a sufficient large, $\Omega(n)$, population size as $n$ the problem size, with a high probability, the $(0,0)$ pattern can be maintained until the optimum is reached, that is, the population in $(\mu+1)$~EA increases the tolerance to the incorrect $(0,1)$ pattern trial.

\subsection{Runtime Analysis of \mpoea on OneMax$_{(0,1^n)}$}
Theorem~\ref{thm:mpoEAOM} only shows the probability that \mpoea can reach the optimum. One further question is about its runtime. Here, we give some comments on the runtime complexity. For the runtime of \mpoea on the original OneMax function, Witt~\cite{Witt06} shows the upper bound of the expected runtime is $O(\mu n+n\log n)$ based on the current best individuals' replicas and fitness increasing. Analogously, for \mpoea on OneMax$_{(0,1^n)}$ function, we could consider the expected time when the number of the current front individuals with $a$ zeros reaches $n/a$, that is, $|m_0|\ge n/a$, and the expected time when a $(0,0)$ pattern offspring with less $a$ zeros is generated for $a>1$ or when one $(0,1)$ pattern individual with all ones is generated for $a=1$ conditional on that there are $n/a$ current front individuals with $a$ zeros. From Lemma~\ref{lem:acctui}, with probability at least $1-\exp\left(-\frac{\delta^2}{2(1+\delta)}(n-1)\right)$, there are at most $\frac12 \mu -1$ possible accumulative temporarily undefeated individuals before the current front individuals have only 1 zero. Hence, in each generation before the current front individuals have only 1 zero, it always holds that at least half of individuals of the whole population are current front individuals and interior individuals. Hence, we could just discuss the population containing no temporarily undefeated individual and twice the upper bound of the expected time to reach the optimum as that for the true process. The $(0,1)$ pattern offspring with $a$ zeros will not influence the evolving process of the current front individuals we focus on until all interior individuals have $a$ zeros. Recalling Lemma~\ref{lem:occupy}, we know that with probability at least $1-\exp(-\tfrac{1}{20}n^{0.5})$, $|m_0| \ge \tfrac{1}{5}n^{0.5}$ if the current phase does not end before all individuals have at most $a$ zeros. Hence, we just need to focus on the case when $\frac{n}{a} \ge \tfrac{1}{5}n^{0.5}$, that is, $a \le 5n^{0.5}$.

We discuss the expected length of the phase defined in Lemma~\ref{lem:occupy}. When $|m_0|<\frac{n}{a}$, we consider the event that one replica of an $m_0$ individual can enter into the next generation. When the population contains interior individual(s) with more than $a$ zeros, the probability is $\frac{|m_0|}{\mu (1-1/n)^n} \ge \frac{|m_0|}{2e\mu}$. When all interior individuals have $a$ zeros, we require $|m_0|$ to be less than $2n/a$. Let $\mu'$ denote the total number of the individuals with $a$ zeros, and we know the probability of event $H$ that one replica of an $m_0$ individual can enter into the next generation is ${\frac{|m_0|}{\mu}\left(1-\frac 1n\right)^n \frac{\mu'+1-|m_0|}{\mu'+1}}$. Note that the probability of event $G$ that an $m_0$ individual generates a $(0,1)$ pattern offspring with $a$ zeros that successfully enters into the next generation is at most $\frac{|m_0|}{\mu}\frac{1}{n}\frac{|m_0|}{\mu'+1}$. Then
\begin{align*}
\frac{\Pr[H]}{\Pr[G]} \ge{}& \frac{\frac{|m_0|}{\mu}\left(1-\frac 1n\right)^n \frac{ {\mu'+1}-|m_0|}{ {\mu'+1}}} {\frac{|m_0|}{\mu}\frac{1}{n}\frac{|m_0|}{ {\mu'+1}}} \ge \frac{1}{2e}\left(\frac{\mu'}{|m_0|}-1\right)n \\
\ge{}& \frac{1}{2e}\left(\frac{ {2(1+\delta)(3e+1)(n+1)}}{|m_0|}-1\right)n \\
\ge{}& \frac{ {(1+\delta)(3e+1)a-1}}{2e}n \ge  {\frac{3}{2}n}
\end{align*}
where the antepenultimate inequality uses $\mu' \ge 2(1+\delta)(3e+1)(n+1)$ and the penultimate inequality uses $|m_0|\le \frac{2n}{a}$. Hence, $\Pr[G \mid H \cup G] \le \frac{2}{3n+2}$. Considering the process merely consisting of $H$ and $G$, let $Z$ be the number that $H$ happens before $G$ happens $\tfrac{1}{5}n^{0.5}$ times. Then $E[Z] \ge \left(\frac{3n+2}{2}-1\right) \tfrac{1}{5}n^{0.5} = \frac{3}{10}n^{1.5}$. It is not difficult to see that $Z+\tfrac{1}{5}n^{0.5}$ stochastically dominates the sum of $\tfrac{1}{5}n^{0.5}$ geometric variables with success probability $\frac{2}{3n+2}$. Hence, with the Chernoff bound for the sum of geometric variables in Lemma~\ref{lem:chernoff}(c), we have that 
\begin{align*}
\Pr\big[&Z+\tfrac{1}{5}n^{0.5}\le 2n+\tfrac{2}{5}n^{0.5}\big] \\
\le{}& \exp\left(-\frac{\left(1-\tfrac{20}{3}n^{-0.5}-\tfrac{4}{3}n^{-1}\right)^2\tfrac15 n^{0.5}}{2-\tfrac43 \left(1-\tfrac{20}{3}n^{-0.5}-\tfrac{4}{3}n^{-1}\right)}\right)\\
\le {}& \exp\left(-\frac{\left(1-2\cdot\tfrac{20e+1}{3e}n^{-0.5}\right)\tfrac15 n^{0.5}}{2-\tfrac43 \left(1-\tfrac{20e+1}{3e}n^{-0.5}\right)}\right)\\
\le {}& \exp\left(-\left(\frac{3}{1+\tfrac{20e+1}{3e}n^{-0.5}}-\frac32\right)\frac15 n^{0.5}\right)\\
\le {}& \exp\left(-\left(\frac{3}{1+\tfrac{20e+1}{3e\cdot 4e}}-\frac32\right)\frac15 n^{0.5}\right)\le \exp\left(-\frac{1}{20}n^{0.5}\right)
\end{align*}
where the second inequality uses the fact $\frac{4}{3n} \le \frac{4}{3\cdot4en^{0.5}}=\frac{1}{3en^{0.5}}$ for $n>(4e)^2$ and the fact $(1-x)^2 \ge 1-2x$ for $x\in\R$, and the penultimate inequality uses $n>(4e)^2$. Since $2n+\tfrac{1}{5}n^{0.5} \ge\frac{2n}{a}+\tfrac{1}{5}n^{0.5}$, we know with probability at least ${1-\exp(-\frac{1}{20}n^{0.5})}$, $|m_0|$ could go above $\frac{2n}{a}$.

When $|m_0|$ goes above $2n/a$, we consider the event $F$ that the current phase ends.
Recalling the proof in Lemma~\ref{lem:occupied}, we know that 
with probability at least ${1-(\frac{e}{e+1})^{\frac{1}{5}n^{0.5}}}$, $F$ happens once before $G$ happens $\frac{n}{a} \ge \frac15 n^{0.5}$ times, thus, before $|m_0|$ goes below $\frac{n}{a}$. 

In summary, in each phase, with probability at least $(1-\exp(-\tfrac{1}{20}n^{0.5}))(1-\exp(-\frac{1}{20}n^{0.5}))(1-(\frac{e}{e+1})^{\frac{1}{5}n^{0.5}}) \ge 1-3\exp(-\tfrac{1}{20}n^{0.5})$, the current front individuals could increase its number to more than $\frac{2n}{a}$, and will remain above $n/a$ afterwards. Hence, together with the runtime analysis of the original \mpoea on OneMax in~\cite{Witt06}, we have the runtime result for \mpoea on OneMax$_{(0,1^n)}$ in the following theorem, and know that comparing with OneMax function, the cost majorly lies on the $o(1)$ success probability for \mpoea solving the time-linkage OneMax$_{(0,1^n)}$, and the asymptotic complexity remains the same for the case when \mpoea is able to find the optimum.

\begin{theorem}
Given any $\delta > 0$. Let $n > (4(1+\delta)e)^2$. Consider using \mpoea with population size ${\mu\ge 4(1+\delta)(3e+1)(n+1)}$ to solve the OneMax$_{(0,1^n)}$ function. Consider the same phase in Lemma~\ref{lem:occupy}. Let $M$ denote the event that 
\begin{itemize}
\item the first generation has at most $\frac34 \mu$ individuals with the $(1,0)$ or $(1,1)$ first bit pattern, at least $\frac 14 \mu$ individuals with the $(0,0)$ pattern, and has all individuals with less than $\frac 34 n$ zeros;
\item there are at most $\frac12 \mu-1$ accumulative temporarily undefeated individuals before the current front individuals only have 1 zero;
\item the number of the current front individual with $a$ zeros can accumulate to $\frac{2n}{a}$ and stay above $\frac{n}{a}$ if the current phase does not end.
\end{itemize} 
Then event $M$ occurring with probability at least $1-(\mu+2) \exp(-\tfrac18 n) -\exp\left(-\frac{\delta^2}{2(1+\delta)}(n-1)\right)-3n\exp(-\tfrac{1}{20}n^{0.5})$, and conditional on $M$, 
 the expected runtime is $O(\mu n)$.
\label{thm:runtimempoEA}
\end{theorem}

Recalling that the expected runtime of $(\mu+1)$~EA on OneMax is $O(\mu n+n\log n)$~\cite{Witt06}, which is $O(n\log n)$ for $\mu=O(\log n)$. Since $\mu=\Omega(n)$ is required for the convergence on OneMax$_{(0,1^n)}$ in Section~\ref{subsec:muoea}, Theorem~\ref{thm:runtimempoEA} shows the expected runtime for OneMax$_{(0,1^n)}$ is $O(n^2)$ if we choose $\mu=\Theta(n)$, which is the same complexity as for OneMax with $\mu=\Theta(n)$. To this degree, comparing $(\mu+1)$~EA solving the time-linkage OneMax$_{(0,1^n)}$ with the original OneMax function, we may say the cost majorly lies on $o(1)$ convergence probability, not the asymptotic complexity.

\section{Conclusion and Future Work}
\label{sec:con}
In recent decades, rigorous theoretical analyses on EAs has progressed significantly. However, despite that many real-world applications have the time-linkage property, that is, the objective function relies on more than one time-step solutions, the theoretical analysis on the fitness function with time-linkage property remains an open problem.


This paper took the first step into this open area. We designed the time-linkage problem OneMax$_{(0,1^n)}$, which considers an opposite preference of the first bit value of the previous time step into the basic OneMax function. Via this problem, we showed that EAs with a population can prevent some stagnation in some deceptive situations caused by the time-linkage property. More specifically, we proved that the simple RLS and \oea cannot reach the optimum of OneMax$_{(0,1^n)}$ with $1-o(1)$ probability but $(\mu+1)$ EA can find the optimum with $1-o(1)$ probability.

The time-linkage OneMax$_{(0,1^n)}$ problem is simple. Only the immediate previous generation and the first bit value of the historical solutions matter for the fitness function. Our future work should consider more complicated algorithms, e.g., with crossover, on more general time-linkage pseudo-Boolean functions, e.g., with more than one bit value and other weight values for the historical solutions, and problems with practical backgrounds. 

\section*{Acknowledgement}
We thank Liyao Gao for his participation and discussion on the \oea part during his summer intern.

\ifCLASSOPTIONcaptionsoff
  \newpage
\fi




%
%
%

%
\begin{IEEEbiography}[{\includegraphics[width=1in,height=1.25in,clip,keepaspectratio]{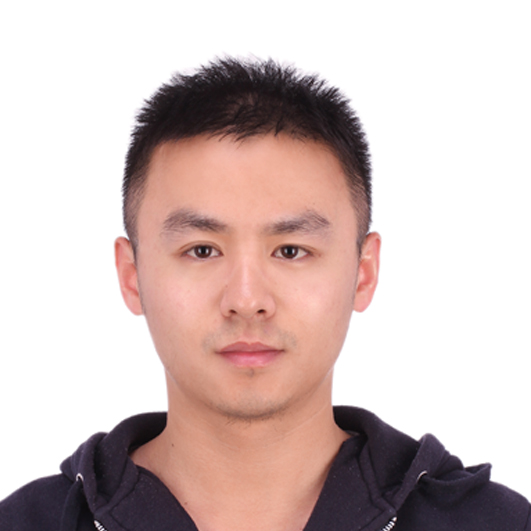}}]{Weijie Zheng}
received Bachelor Degree (July 2013) in Mathematics and Applied Mathematics from Harbin Institute of Technology, Harbin, Heilongjiang, China, and Doctoral Degree (October 2018) in Computer Science and Technology from Tsinghua University, Beijing, China. 

He is now a postdoctoral researcher in the Department of Computer Science and Engineering at Southern University of Science and Technology, Shenzhen, Guangdong, China, and in the School of Computer Science and Technology, University of Science and Technology of China, Hefei, Anhui, China.

His current research majorly focuses on the theoretical analysis and design of evolutionary algorithms, especially binary differential evolution and estimation-of-distribution algorithms.
\end{IEEEbiography}

\begin{IEEEbiography}[{\includegraphics[width=1in,height=1.25in,clip,keepaspectratio]{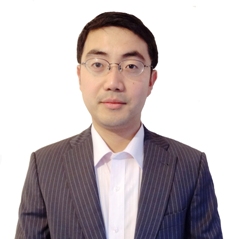}}]{Huanhuan Chen}
received the B.Sc. degree from the University of Science and Technology of China (USTC), Hefei, China, in 2004, and the Ph.D. degree in computer science from the University of Birmingham, Birmingham, U.K., in 2008.

He is currently a Full Professor with the School of Computer Science and Technology, USTC. His current research interests include neural networks, Bayesian inference, and evolutionary computation.

Dr. Chen was a recipient of the 2015 International Neural Network Society Young Investigator Award, the 2012 IEEE Computational Intelligence Society Outstanding Ph.D. Dissertation Award, the IEEE TRANSACTIONS ON NEURAL NETWORKS Outstanding Paper Award, and the 2009 British Computer Society Distinguished Dissertations Award. He is an Associate Editor of the IEEE TRANSACTIONS ON NEURAL NETWORKS AND LEARNING SYSTEMS and the IEEE TRANSACTIONS ON EMERGING TOPICS IN COMPUTATIONAL INTELLIGENCE.
\end{IEEEbiography}

\begin{IEEEbiography}[{\includegraphics[width=1in,height=1.25in,clip,keepaspectratio]{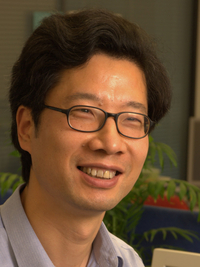}}]{Xin Yao}
obtained his Ph.D. in 1990 from the University of Science and Technology of China (USTC), MSc in 1985 from North China Institute of Computing Technologies and BSc in 1982 from USTC. 

He is a Chair Professor of Computer Science at the Southern University of Science and Technology, Shenzhen, China, and a part-time Professor of Computer Science at the University of Birmingham, UK. He is an IEEE Fellow and was a Distinguished Lecturer of the IEEE Computational Intelligence Society (CIS). He was the President (2014-15) of IEEE CIS and the Editor-in-Chief (2003-08) of IEEE Transactions on Evolutionary Computation. His major research interests include evolutionary computation, ensemble learning, and their applications to software engineering. His work won the 2001 IEEE Donald G. Fink Prize Paper Award; 2010, 2016 and 2017 IEEE Transactions on Evolutionary Computation Outstanding Paper Awards; 2011 IEEE Transactions on Neural Networks Outstanding Paper Award; and many other best paper awards at conferences. He received a prestigious Royal Society Wolfson Research Merit Award in 2012, the IEEE CIS Evolutionary Computation Pioneer Award in 2013 and the 2020 IEEE Frank Rosenblatt Award.
\end{IEEEbiography}


%
%




\end{document}